\documentclass[sigplan,nonacm]{acmart}
\usepackage{amsmath}
\usepackage{stmaryrd}                     
\usepackage{braket}
\usepackage{multirow}
\usepackage{enumitem}
\usepackage{placeins}
\usepackage{tikz}
\usetikzlibrary{arrows.meta,positioning,shapes.misc,calc,fit,backgrounds}
\usepackage[EULERGREEK]{sansmath}
\SetMathAlphabet{\mathrm}{sans}{OT1}{phv}{m}{n}
\SetMathAlphabet{\mathsfsl}{sans}{OT1}{phv}{m}{sl}
\SetMathAlphabet{\mathsfbf}{sans}{OT1}{phv}{b}{n}
\SetSymbolFont{operators}{sans}{OT1}{phv}{m}{n}
\SetSymbolFont{symbols}{sans}{LMS}{ntxsy}{m}{n}
\SetSymbolFont{letters}{sans}{OML}{nxlmi}{m}{it}
\definecolor{fsCircuitBlue}{HTML}{315A77}
\definecolor{fsCircuitRust}{HTML}{A34F32}
\definecolor{fsCircuitInk}{HTML}{24282B}
\newcommand{\fscircuitstyle}{%
  \fontfamily{phv}\selectfont\sansmath
  \colorlet{fsBlue}{fsCircuitBlue}%
  \colorlet{fsRed}{fsCircuitRust}%
  \tikzset{every node/.append style={text=fsCircuitInk}}}
\definecolor{fsBlue}{HTML}{2879B9}
\definecolor{fsRed}{HTML}{C84843}
\definecolor{fsGreen}{HTML}{44876B}
\definecolor{fsGray}{HTML}{77828B}
\tikzset{fs flow/.style={-{Latex[length=2mm]},semithick},
  fs faint/.style={draw=fsGray!45,line width=.35pt}}
\newcommand{\fscnot}[4]{%
  \draw[#4,line width=.65pt] ({#1},{#2}) -- ({#1},{#3});
  \fill[#4] ({#1},{#3}) circle (1.65pt);
  \draw[#4,fill=white,line width=.65pt] ({#1},{#2}) circle (2.65pt);
  \draw[#4,line width=.65pt] ({#1-.093},{#2}) -- ({#1+.093},{#2});
  \draw[#4,line width=.65pt] ({#1},{#2-.093}) -- ({#1},{#2+.093});}

\newcommand{\AS}{AlphaSyndrome}
\newcommand{\PH}{PropHunt}
\newcommand{\code}[3]{\ensuremath{\llbracket #1,#2,#3\rrbracket}}
\newcommand{\LER}{\ensuremath{\mathrm{LER}}}
\newcommand{\ESS}{\ensuremath{\mathrm{ESS}}}
\newcommand{\Nmax}{\ensuremath{N_{\max}}}
\newcommand{\sysname}{FastSched}

\newcommand{\meanRed}{25.9}
\newcommand{\peakRed}{49.4}

\begin{document}

\title{Reinforcement Learning for Syndrome Extraction}

\author{John Zhuoyang Ye}
\affiliation{%
  \institution{University of California, Los Angeles}
  \city{Los Angeles}
  \country{United States}
}
\email{yezhuoyang@cs.ucla.edu}

\author{Aarav Pabla}
\affiliation{%
  \institution{University of California, Los Angeles}
  \city{Los Angeles}
  \country{United States}
}
\email{apabla@ucla.edu}

\author{Jens Palsberg}
\affiliation{%
  \institution{University of California, Los Angeles}
  \city{Los Angeles}
  \country{United States}
}
\email{palsberg@cs.ucla.edu}

\renewcommand{\shortauthors}{Ye, Pabla, and Palsberg}

\makeatletter
\patchcmd{\ccsdesc@parse}{\else; }{\else, }{}{\PackageError{main}{CCS punctuation patch failed}{Check the acmart definition.}}
\makeatother



\begin{abstract}

A key subtask of quantum error correction is to extract a syndrome that, if nontrivial, signals an error.
The number of possible ways to extract a syndrome grows exponentially with the syndrome size, and these implementations vary greatly in fault tolerance, as measured by their logical error rates.
This creates a natural search problem: find an implementation with a low logical error rate.
Previous work solves this problem but sacrifices either solution quality or scalability.
In this paper, we use reinforcement learning and importance sampling to outperform previous work at all scales.
Compared with the state of the art automatic scheduling tools \AS{} and \PH{}, our tool reduces the logical error rate by \meanRed\% and 71.7\% on average, respectively, culminating with a reduction of 97.8\% for a surface code with distance 15.

\end{abstract}

\maketitle

\section{Introduction}\label{sec:intro}

\paragraph{Syndrome Extraction.}

Quantum computations that factor integers and simulate molecules require reliable logical qubits~\cite{Gidney2025rsa,Babbush2018spectra}.
Quantum error correction (QEC) provides this reliability by encoding logical information redundantly in unreliable physical qubits and repeatedly extracting syndromes to detect and correct errors~\cite{Shor1995,Terhal2015}.
The required level of reliability depends on the application and can be, for example, $10^{-15}$ errors per logical computation step~\cite[Section~3.2]{Gidney2025rsa}.
In this paper, we focus on the reliability of syndrome extraction.
Evaluating circuits with low logical error rates (LERs) requires many
samples because failures are rare. Figure~\ref{fig:intro-contours} compares
measured LERs as the surface-code distance increases.
\sysname{} finds schedules with lower LERs across the tested distances.

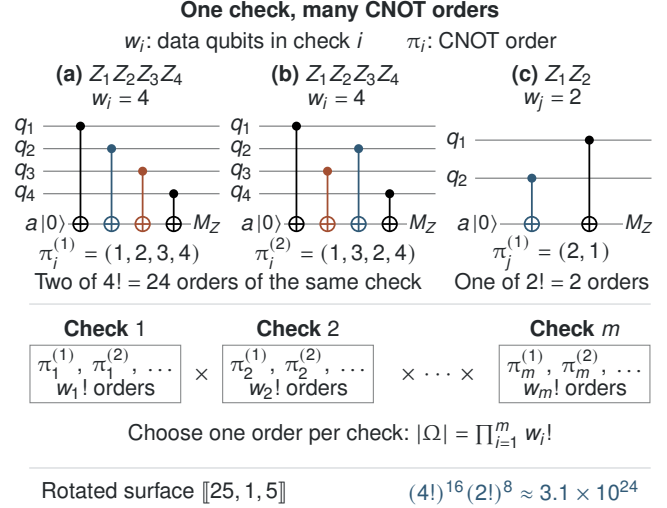
\begin{figure}[t]
\centering
\begingroup
\fscircuitstyle
\resizebox{\linewidth}{!}{%
\begin{tikzpicture}[x=1cm,y=1cm,font=\footnotesize,
  card/.style={draw=fsCircuitInk!55,fill=white,line width=.4pt},
  wire/.style={draw=fsCircuitInk!65,line width=.4pt}]
\path[use as bounding box] (0,.22) rectangle (8.24,-6.37);
\node[font=\footnotesize\bfseries] at (4.12,0) {One check, many CNOT orders};
\node at (4.12,-.40)
  {$w_i$: data qubits in check $i$\qquad $\pi_i$: CNOT order};

\foreach \shift/\letter/\second/\third/\order/\choice in
  {0/a/2/3/{1,2,3,4}/1,2.78/b/3/2/{1,3,2,4}/2}{
  \begin{scope}[shift={(\shift,-1.48)}]
    \node[font=\footnotesize\bfseries] at (1.29,.65)
      {(\letter) $Z_1Z_2Z_3Z_4$};
    \node at (1.29,.34) {$w_i=4$};
    \foreach \q in {1,...,4}{
      \node[anchor=east,inner sep=.5pt] at (.24,{-.29*(\q-1)}) {$q_{\q}$};
      \draw[wire] (.30,{-.29*(\q-1)}) -- (2.57,{-.29*(\q-1)});}
    \node[anchor=east,inner sep=.5pt] at (.24,-1.26) {$a$};
    \draw[wire] (.30,-1.26) -- (2.57,-1.26);
    \node[fill=white,inner sep=.4pt,font=\scriptsize] at (.43,-1.26) {$\ket{0}$};
    \node[fill=white,inner sep=.4pt,font=\scriptsize] at (2.41,-1.26) {$M_Z$};
    \fscnot{.78}{-1.26}{0}{black}
    \foreach \t/\q in {1/\second,2/\third}{
      \ifnum\q=2\relax\def\gatecolor{fsBlue}\else\def\gatecolor{fsRed}\fi
      \fscnot{.78+.40*\t}{-1.26}{-.29*(\q-1)}{\gatecolor}}
    \fscnot{1.98}{-1.26}{-.87}{black}
    \node at (1.29,-1.63) {$\pi_i^{(\choice)}=(\order)$};
  \end{scope}}
\begin{scope}[shift={(5.56,-1.48)}]
  \node[font=\footnotesize\bfseries] at (1.29,.65) {(c) $Z_1Z_2$};
  \node at (1.29,.34) {$w_j=2$};
  \foreach \q/\y in {1/-.18,2/-.67}{
    \node[anchor=east,inner sep=.5pt] at (.24,\y) {$q_{\q}$};
    \draw[wire] (.30,\y) -- (2.57,\y);}
  \node[anchor=east,inner sep=.5pt] at (.24,-1.26) {$a$};
  \draw[wire] (.30,-1.26) -- (2.57,-1.26);
  \node[fill=white,inner sep=.4pt,font=\scriptsize] at (.43,-1.26) {$\ket{0}$};
  \node[fill=white,inner sep=.4pt,font=\scriptsize] at (2.41,-1.26) {$M_Z$};
  \fscnot{1.03}{-1.26}{-.67}{fsBlue}
  \fscnot{1.77}{-1.26}{-.18}{black}
  \node at (1.29,-1.63) {$\pi_j^{(1)}=(2,1)$};
\end{scope}
\node at (2.68,-3.47) {Two of $4!=24$ orders of the same check};
\node at (6.85,-3.47) {One of $2!=2$ orders};
\draw[fsGray!35,line width=.4pt] (.12,-3.77) -- (8.12,-3.77);

\foreach \x/\i in {1.10/1,3.58/2,7.14/m}{
  \node[font=\footnotesize\bfseries] at (\x,-4.08) {Check $\i$};
  \draw[card] ({\x-.97},-4.30) rectangle ({\x+.97},-5.04);
  \node at (\x,-4.54) {$\pi_{\i}^{(1)},\ \pi_{\i}^{(2)},\ \ldots$};
  \node at (\x,-4.86) {$w_{\i}!$ orders};}
\node at (2.34,-4.67) {$\times$};
\node at (5.38,-4.67) {$\times\,\cdots\,\times$};
\node[inner xsep=10pt,inner ysep=4pt]
  at (4.12,-5.46)
  {Choose one order per check: $|\Omega|=\prod_{i=1}^{m}w_i!$};
\draw[fsGray!35,line width=.4pt] (.12,-5.88) -- (8.12,-5.88);
\node[anchor=west] at (.15,-6.19) {Rotated surface $\code{25}{1}{5}$};
\node[anchor=east,text=fsBlue] at (8.10,-6.19)
  {$(4!)^{16}(2!)^8\approx\mathbf{3.1\times10^{24}}$};
\end{tikzpicture}}
\endgroup
\caption{Syndrome extraction example.}  
\Description{The check weight w counts data qubits, and pi lists the
  qubits in CNOT time order from left to right. The first two circuits
  measure Z one Z two Z three Z four with orders one two three four and
  one three two four, illustrating two of twenty-four permutations.
  The third circuit measures a weight-two boundary check in order two one,
  one of two permutations. Each data qubit controls a CNOT targeting an
  ancilla prepared in zero and measured in Z. Choosing one permutation
  for each of m checks gives the product of their factorial counts.
  Sixteen weight-four interior checks and eight weight-two boundary
  checks of the distance-five rotated surface code give about 3.1 times
  ten to the twenty-four orders.}
\label{fig:scheduling-space}
\label{fig:hook}
\end{figure}

A syndrome is a bit pattern obtained by a sequence of checks, each of which measures a stabilizer and produces one syndrome bit. 
Each check uses an ancilla qubit to collect information from several data qubits through CNOT gates.
Figure~\ref{fig:hook} illustrates a syndrome extractor for the rotated surface code $\code{25}{1}{5}$, which encodes a single logical qubit into 25 physical qubits.  Figure~\ref{fig:hook}(a,b) shows two possible CNOT orders for the stabilizer $Z_1Z_2Z_3Z_4$.
In the absence of noise, changing the CNOT order within a check preserves its behavior.
During a noisy execution, however, an error on the ancilla can propagate through the remaining CNOTs to the data qubits, resulting in \emph{hook errors}~\cite{Tomita_2014,Beverland2024}.

A \emph{schedule} specifies the CNOT order for each check and the time step of each gate.
For $m$ checks, let $w_i$ be the number of data qubits in check $i$.
Check $i$ has $w_i!$ CNOT orders, so choosing one order per check gives $\prod_{i=1}^{m} w_i!$ combinations before imposing timing restrictions or exploiting symmetries.
The code in Figure~\ref{fig:hook} has sixteen weight-four checks and eight weight-two checks, giving $(4!)^{16}(2!)^8\approx3.1\times10^{24}$ order combinations.

The schedules differ in how they propagate errors and, consequently, in their logical error rate (LER).
For example, prior work reports up to a 96.2\% LER reduction by optimizing the schedule relative to a depth-optimal baseline~\cite[Section~5.2]{AlphaSyndrome2026}.
We therefore seek a schedule with low LER:
\begin{quotation}
{\em The schedule-synthesis problem:\\
find a schedule with a low logical error rate.}
\end{quotation}

\paragraph{State of the Art.}

\AS{}~\cite{AlphaSyndrome2026} and \PH{}~\cite{PropHunt2026} demonstrate
the potential of automatic schedule synthesis.
\AS{} uses Monte Carlo tree search (MCTS), and its schedules match Google's
hand-designed surface-code schedules. Impressively, for the \code{72}{12}{6}
bivariate bicycle code, \AS{} outperforms IBM's schedule, and for different codes, it comes with results for physical error probabilities down to $10^{-5}$ on three code instances~\cite[Section~5.6]{AlphaSyndrome2026}.
\PH{} uses maximum satisfiability (MaxSAT) to find fault patterns that a
decoder (which infers corrections from syndrome information) can confuse, and then changes the circuit to resolve those ambiguities. 
Impressively, \PH{} matches hand-designed surface-code performance and reduces LER by
$2.5$--$4\times$ relative to coloration circuits on its tested lifted-product
and random quantum Tanner codes at physical error probability
$10^{-3}$~\cite[Section~6.1]{PropHunt2026}.

Together, these approaches expose a tradeoff between solution quality and scalability: \AS{} finds high-quality schedules but scales poorly to codes with large distances, while \PH{} scales to larger codes but can produce schedules with much higher LERs.
In this paper, we address the schedule-synthesis problem without making this tradeoff: our approach combines quality with scalability.

\paragraph{Our Results.}

\begin{figure}[t]
  \centering
  \begingroup
  \begin{tikzpicture}[
    x=\linewidth,y=1cm,
    font=\sffamily\fontsize{8}{9.5}\selectfont,
    every node/.style={align=center,inner sep=0pt,outer sep=0pt},
    heading/.style={font=\sffamily\bfseries\fontsize{8}{9.5}\selectfont}
  ]
    \path[use as bounding box] (0,-.05) rectangle (1,4.29);

    \node[heading,text width=.55\linewidth] at (.721,4.09)
      {\em How to evaluate schedules};
    \node[heading,text width=.25\linewidth] at (.580,3.57)
      {Monte Carlo \\ sampling};
    \node[heading,text width=.25\linewidth] at (.858,3.57)
      {Importance \\ sampling};

    \node[heading,text width=.17\linewidth] at (.085,1.61)
      {\em How to choose\\schedules};
    \node[heading,text width=.25\linewidth] at (.310,2.68)
      {MCTS};
    \node[heading,text width=.25\linewidth] at (.310,1.61)
      {MaxSAT};
    \node[heading,text width=.25\linewidth] at (.310,.54)
      {Reinforcement \\ learning};

    \draw[line width=.75pt] (.444,0) rectangle (.997,3.21);
    \node[text width=.25\linewidth] at (.580,2.68)
      {AlphaSyndrome\\\cite{AlphaSyndrome2026}};
    \node[text width=.25\linewidth] at (.858,2.68)
      {\textemdash}; 
    \node[text width=.25\linewidth] at (.580,1.61)
      {PropHunt~\cite{PropHunt2026}}; 
    \node[text width=.25\linewidth] at (.858,1.61)
      {\textemdash};
    \node[text width=.25\linewidth] at (.580,.54)
      {\textemdash}; 
    \node[text width=.25\linewidth] at (.858,.54)
      {\sysname{} \\ {[}this paper{}]};
  \end{tikzpicture}
  \endgroup
  \caption{Two dimensions of schedule synthesis.} 
  \Description{A three-row, two-column matrix with one outer rectangle and no
    internal dividers. Column labels distinguish direct Monte Carlo sampling from
    importance sampling. The first row, MCTS, contains AlphaSyndrome and
    cheaper LER evaluation as a possible extension. The second row,
    MaxSAT, contains PropHunt with direct sampling used for LER validation.
    A dash indicates no importance-sampling method is shown. PropHunt's
    search instead uses ambiguity checks. The third row, reinforcement
    learning, contains searching for lower LER and FastSched combining both
    contributions. The matrix describes design choices and objectives,
    not six measured comparisons.}
  \label{fig:two-dimensions}
\end{figure}

We introduce \sysname{}, which uses reinforcement learning (RL) to search for schedules and importance sampling to evaluate them.  Figure~\ref{fig:two-dimensions} illustrates these choices and what other tools use.
Compared with the state of the art automatic scheduling tools \AS{} and \PH{}, our tool reduces the logical error rate by \meanRed\% and 71.7\% on average, respectively, culminating with a reduction of 97.8\% for a surface code with distance 15.
Figure~\ref{fig:intro-contours} illustrates this comparison, with additional data from our evaluation.
The distance study matches the code, noise model, and decoder.
Search caps are 10 minutes through $d=9$ and 30 minutes thereafter.



\paragraph{Outline of Our Approach.}

Our RL agent learns which gate orders produce a low LER and ultimately selects a single schedule. Specifically, it learns from LER estimates for complete circuits and chooses each check's CNOT order based on earlier choices. During training, our RL agent evaluates thousands of schedules using small sample budgets. After training, it spends larger budgets on candidate re-evaluation to distinguish the retained candidates (Section~\ref{sec:selection}). Concentrating precise evaluation on a few finalists reduces the cost relative to evaluating every training candidate at that precision. Finally, we perform a high-precision LER evaluation of the selected schedule.

Our schedule evaluator builds on rare-event methods for QEC~\cite{BravyiVargo2013rare,Beverland2025failfast,Mayer2025rare}.
This is in contrast to both \AS{} and \PH{}, which use Monte Carlo methods to estimate LERs.
The rare event is a logical error: although the LER depends on the schedule, noise model, and decoder~\cite{AlphaSyndrome2026,PropHunt2026}, most samples do not result in a logical failure.
Thus, we face an instance of \emph{rare-event estimation}, which estimates the probability of outcomes that Monte Carlo sampling rarely observes.
We use importance sampling, which amplifies fault probabilities to make logical failures more frequent and corrects for the changed sampling distribution.

\begin{figure}[t]
  \centering
  \includegraphics[width=\columnwidth]{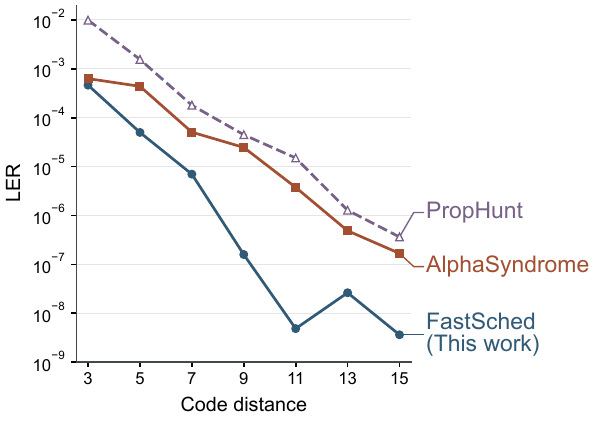}
  \caption{Surface-code LER versus distance under Brisbane noise
    (lower is better). Lines join point estimates;
    Figure~\ref{fig:surface-distance} adds confidence bounds.
    The dashed \PH{} curve shows starting circuits because its searches
    returned no optimized schedules.}
  \Description{Three curves show measured logical error rates at rotated
    surface-code distances three, five, seven, nine, eleven, thirteen,
    and fifteen. The horizontal axis is code distance. The vertical
    axis is LER on a logarithmic scale. Blue circles mark FastSched,
    rust squares mark AlphaSyndrome, and a purple dashed line with open
    triangles marks PropHunt initialization-only circuits. Names appear
    beside the curve endpoints in larger type. FastSched is labeled
    FastSched (This work). All results come from our matched
    Brisbane study with one noisy extraction round, ideal boundaries,
    and PyMatching. LER is the sum of X-memory and Z-memory failure
    rates. Lines join point estimates. Evaluation reports confidence
    bounds. The curves have no numerical point annotations.}
  \label{fig:intro-contours}
\end{figure}


\paragraph{Our Contributions.}
\begin{itemize}[leftmargin=*]
    \item An RL formulation of schedule synthesis that finds lower-LER schedules than 
    the state of the art.
    \item A schedule evaluator based on importance sampling and increasing sample budgets.
    \item A proof that our schedule evaluator is unbiased under a fixed sampling budget.
\end{itemize}

\paragraph{Rest of the Paper.}
In Section~\ref{sec:background} we recall the concepts that we use in the paper.  In Section~\ref{sec:method}, we present our RL agent, and in Section~\ref{sec:amplification}, we present our approach to evaluating schedules.  Finally, in Section~\ref{sec:eval}, we give our evaluation, and in Section~\ref{sec:related}, we discuss related work.

\section{Background}
\label{sec:background}

\paragraph{Stabilizer Codes and Syndrome Extraction}

A stabilizer code is the joint $+1$ eigenspace of a commuting Pauli
group excluding $-I$~\cite{Gottesman_1997}. We call its measured generators
\emph{checks}. An \code{n}{k}{d} code encodes $k$ logical qubits in $n$
data qubits. Distance $d$ is the minimum weight of a data-qubit Pauli
operator that preserves the codespace and changes the logical information.
A Pauli error flips the outcome of each check with which it anticommutes.
The check outcomes form the \emph{syndrome}.

To measure an $X$ check, an ancilla is prepared in $\ket{+}$, acts as
the control of CNOTs onto the check's data qubits, and is measured in
the $X$ basis. A $Z$ check uses an ancilla prepared in $\ket{0}$ as
the CNOT target and measures it in the $Z$ basis~\cite{NielsenChuang2010}.
Within either check, the CNOT order preserves the ideal measurement.
For overlapping $X$ and $Z$ checks, a valid interleaving must also
preserve the joint measurement~\cite{geher2024tangling}.
This can be done by placing all $X$-check
CNOTs before all $Z$-check CNOTs, as illustrated in Figure~\ref{fig:smc}.

Calderbank--Shor--Steane (CSS) codes have checks made entirely of $X$
operators or entirely of $Z$ operators~\cite{CalderbankShor1996,PhysRevLett.77.793}.
Figure~\ref{fig:smc}(a,b)
shows two orders of one weight-four $X$ check. Panel (c) shows how
separate ancillas measure overlapping $X$ and $Z$ checks.

\paragraph{From an Ancilla Fault to a Logical Failure}
\label{sec:hook-background}

\begin{figure}[t]
\centering
\begingroup
\fscircuitstyle
\resizebox{\linewidth}{!}{%
\begin{tikzpicture}[x=1cm,y=1cm,font=\footnotesize,
  wire/.style={draw=fsCircuitInk!65,line width=.4pt}]
\newcommand{\hookframe}{%
  \foreach \r in {1,...,4}{
    \node[anchor=east] at (.32,{-.35*(\r-1)}) {$q_{\r}$};
    \draw[wire] (.44,{-.35*(\r-1)}) -- (3.83,{-.35*(\r-1)});}
  \node[anchor=east] at (.23,-1.54) {$a_X$};
  \draw[wire] (.44,-1.54) -- (3.83,-1.54);
  \node[fill=white,inner sep=.5pt] at (.60,-1.54) {$\ket{+}$};
  \node[fill=white,inner sep=.5pt] at (3.65,-1.54) {$M_X$};}
\foreach \shift/\letter/\order/\second/\third/\hook in
  {0/a/{1,2,3,4}/2/3/{X_3X_4},4.23/b/{1,3,2,4}/3/2/{X_2X_4}}{
  \begin{scope}[shift={(\shift,0)}]
    \node[font=\footnotesize\bfseries] at (2.04,.86)
      {(\letter) Same check $X_1X_2X_3X_4$};
    \node at (2.04,.47) {Order $(\order)$};
    \node[font=\scriptsize,text=fsRed,anchor=south] at (3.66,.12) {Hook};
    \hookframe
    \fscnot{1.00}{0}{-1.54}{black}
    \fscnot{1.56}{-.35*(\second-1)}{-1.54}{black}
    \draw[fsRed,line width=1pt] (1.88,-1.54) -- (3.14,-1.54);
    \draw[fsRed,line width=1pt] (2.34,{-.35*(\third-1)}) -- (3.66,{-.35*(\third-1)});
    \draw[fsRed,line width=1pt] (2.98,-1.05) -- (3.66,-1.05);
    \fscnot{2.34}{-.35*(\third-1)}{-1.54}{fsRed}
    \fscnot{2.98}{-1.05}{-1.54}{fsRed}
    \node[draw=fsRed,fill=white,inner sep=1pt,text=fsRed] at (1.90,-1.54) {$X$};
    \node[fill=white,inner sep=.5pt,text=fsRed] at (3.66,{-.35*(\third-1)}) {$X$};
    \node[fill=white,inner sep=.5pt,text=fsRed] at (3.66,-1.05) {$X$};
    \node[text=fsRed] at (2.04,-1.99) {Data error: $\hook$};
  \end{scope}}
\draw[black!25,line width=.4pt] (.12,-2.41) -- (8.20,-2.41);
\begin{scope}[shift={(0,-3.20)}]
  \node[anchor=west,font=\footnotesize\bfseries] at (.08,.39)
    {(c) Overlapping $X$ and $Z$ checks};
  \node[anchor=east,font=\scriptsize] at (8.18,.39)
    {$X_1X_2X_3X_4$, $Z_1Z_2Z_3Z_4$};
  \foreach \r in {1,...,4}{
    \node[anchor=east] at (.44,{-.34*(\r-1)}) {$q_{\r}$};
    \draw[wire] (.55,{-.34*(\r-1)}) -- (8.05,{-.34*(\r-1)});}
  \foreach \y/\lab/\prep/\basis in {-1.58/a_X/+/X,-2.14/a_Z/0/Z}{
    \node[anchor=east] at (.44,\y) {$\lab$};
    \draw[wire] (.55,\y) -- (8.05,\y);
    \node[fill=white,inner sep=.5pt] at (.73,\y) {$\ket{\prep}$};
    \node[fill=white,inner sep=.5pt] at (7.87,\y) {$M_{\basis}$};}
  \foreach \r in {1,...,4}{
    \fscnot{1.25+.59*(\r-1)}{-.34*(\r-1)}{-1.58}{fsRed}
    \fscnot{4.36+.59*(\r-1)}{-2.14}{-.34*(\r-1)}{fsBlue}}
  \node[text=fsRed] at (2.14,-2.55) {$a_X\to$ data};
  \node[text=fsBlue] at (5.25,-2.55) {data $\to a_Z$};
\end{scope}
\draw[black!25,line width=.4pt] (.12,-6.17) -- (8.20,-6.17);
\foreach \shift/\letter/\title in
  {0/d/{Transverse hook},4.23/e/{Aligned hook}}{
  \begin{scope}[shift={(\shift,-7.00)}]
    \node[font=\footnotesize\bfseries] at (2.04,.40)
      {(\letter) \title};
    \node[font=\scriptsize] at (2.04,.06) {Rotated surface code, $d=3$};
    \fill[fsRed!9] (.78,-.48) rectangle (1.54,-1.24);
    \fill[fsRed!9] (1.54,-1.24) rectangle (2.30,-2.00);
    \fill[fsBlue!9] (1.54,-.48) rectangle (2.30,-1.24);
    \fill[fsBlue!9] (.78,-1.24) rectangle (1.54,-2.00);
    \foreach \x in {.78,1.54,2.30}{
      \draw[wire] (\x,-.48) -- (\x,-2.00);}
    \foreach \y in {-.48,-1.24,-2.00}{
      \draw[wire] (.78,\y) -- (2.30,\y);}
    \draw[fsBlue,dashed,line width=1pt] (1.54,-.23) -- (1.54,-2.22);
    \node[anchor=west,text=fsBlue,font=\scriptsize] at (2.49,-.65)
      {$\bar X=X_2X_4X_5$};
    \foreach \x in {.78,1.54,2.30}{
      \foreach \y in {-.48,-1.24,-2.00}{\fill[black] (\x,\y) circle (1.8pt);}}
    \foreach \x/\y/\q in {.78/-.48/1,1.54/-.48/2,.78/-1.24/3,1.54/-1.24/4,1.54/-2.00/5}{
      \node[anchor=south east,inner sep=1.5pt,font=\scriptsize] at (\x,\y) {$q_{\q}$};}
    \ifdim\shift pt=0pt
      \draw[fsRed,line width=2.4pt] (.78,-1.24) -- (1.54,-1.24);
      \node[anchor=west,text=fsRed,font=\scriptsize] at (2.49,-1.25) {$X_3X_4$};
      \node[font=\scriptsize] at (2.04,-2.61) {Hook crosses the logical string};
    \else
      \draw[fsRed,line width=2.4pt] (1.54,-.48) -- (1.54,-1.24);
      \node[anchor=west,text=fsRed,font=\scriptsize] at (2.49,-1.25) {$X_2X_4$};
      \draw[fsRed,line width=1pt] (1.54,-2.00) circle (3.3pt);
      \node[anchor=west,text=fsRed,font=\scriptsize,align=left] at (2.49,-1.94)
        {Second fault:\\$X_5$};
      \node[font=\scriptsize] at (2.04,-2.61) {Two ancilla faults produce $\bar X$};
    \fi
  \end{scope}}
\end{tikzpicture}}
\endgroup
\caption{Gate order and fault propagation.}
\Description{Two circuits measure X one X two X three X four with CNOT
  orders one two three four and one three two four. An X fault after the
  second CNOT propagates to data qubits three and four in the first
  circuit and two and four in the second. A Hook label sits above each
  pair of propagated X errors. A third circuit shows both
  X-type and Z-type checks on four shared data qubits. The X ancilla
  controls data targets, and the data qubits control the Z ancilla.
  The two lower panels show a three-by-three data grid. The middle column
  carries logical X two X four X five. The horizontal hook in panel d
  crosses this string. The vertical hook in panel e covers its first two
  sites, and a second ancilla fault supplies X five at the third site.}
\label{fig:smc}
\label{fig:surface-codes}
\end{figure}
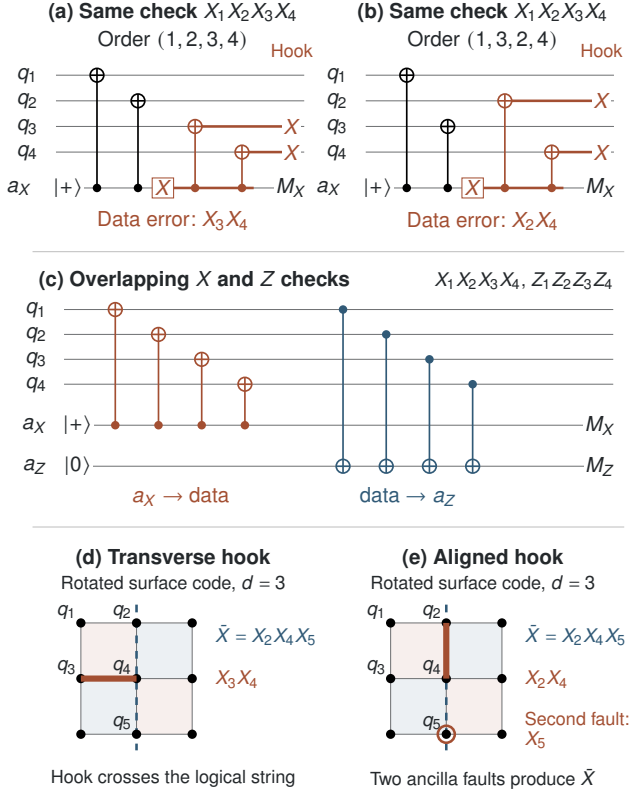

An ancilla interacts with several data qubits, so one ancilla fault can
cause several data errors. CNOT propagates an $X$ error on its control
to its target, and a $Z$ error on its target to its
control~\cite{NielsenChuang2010}. In Figure~\ref{fig:smc}(a), an ancilla
$X$ fault after the second CNOT spreads to the last two targets,
producing $X_3X_4$. Changing the order to $(1,3,2,4)$ in panel (b)
produces $X_2X_4$ from the same fault. These correlated data errors are
\emph{hook errors}~\cite{Tomita_2014,Beverland2024}.

To assess a hook's effect on the encoded state, data errors are considered
up to stabilizer equivalence. Multiplying a data error by a stabilizer
preserves its action on the codespace~\cite{Gottesman_1997}.
For example, an ancilla $X$ fault between the first and second CNOTs
in Figure~\ref{fig:smc}(a) produces $X_2X_3X_4$ on the data.
Multiplication by the check $X_1X_2X_3X_4$ gives $X_1$.
The resulting data error is therefore equivalent to a single-qubit error.
For circuit-level evaluation, we also track the measurement outcomes
produced by each fault~\cite{Beverland2024}.

The rotated distance-three surface code makes the consequence concrete.
In Figure~\ref{fig:smc}(d,e), the blue vertical string
$\bar X=X_2X_4X_5$ is a weight-three logical operator.
The horizontal hook $X_3X_4$ in panel (d) crosses that string.
The vertical hook $X_2X_4$ in panel (e) supplies two of its three supports.
A second ancilla fault in a boundary check can propagate to $q_5$ alone.
Together, the two faults produce $\bar X$. The code distance is three,
but this aligned schedule admits a logical error from two circuit
faults~\cite{Tomita_2014,Beverland2024}.

\emph{Circuit fault distance} is the minimum number of elementary
faults whose combined effect is an undetectable nontrivial logical
operation for a specified circuit and fault model~\cite{Beverland2024}.
The aligned example gives an upper bound of two. A transverse hook
avoids that particular two-fault construction. Determining the distance
of a complete schedule requires checking all allowed faults.
Schedules with the same circuit fault distance can still have different LERs.

To compare such schedules, we evaluate their LERs under a fixed noise model
and decoder. LER sums the probabilities of all fault patterns that cause
decoder failure. CNOT order changes how faults propagate and which patterns
lead the decoder to choose an incorrect correction~\cite{AlphaSyndrome2026}.
Two schedules can therefore have the same minimum fault count for an
undetectable logical error while differing in the number and probability
of patterns that cause decoding failure.
For example, simulations of the unrotated surface code show that changing
CNOT order changes LER under the same noise model and minimum-weight
perfect matching decoder~\cite[Section~IV.A]{ORourke2025compare}.

\paragraph{Syndrome-Extraction Schedules}

In this paper, we focus on the \emph{single-ancilla syndrome extraction} previously described, where each check has one dedicated ancilla that gets entangled with its data qubits and is then measured. Different checks can operate in parallel using their own ancillas.

A schedule assigns CNOTs to time steps, with at most one gate per qubit in a step. The CNOTs within each check can be ordered in different ways, and the choices across checks affect error propagation, parallelism, circuit depth, and idle noise.

For $m$ checks, where check $i$ has weight $w_i$, the unrestricted order space contains
$\prod_{i=1}^m w_i!$
combinations before exploiting symmetries or imposing restrictions.
For example, the \code{7}{1}{3} Steane code has six weight-four checks, giving
$24^6=191{,}102{,}976$ orderings.
A distance-five rotated surface code has sixteen weight-four and eight weight-two checks, giving
$24^{16}2^8\approx3.1\times10^{24}$ orderings.

\paragraph{Logical Error Rate and Circuit Evaluation}
\label{sec:score-background}

A memory experiment compares logical observables before and after
syndrome extraction. A \emph{detector} is a parity of measurements whose
value is deterministic without faults. A detector can compare outcomes
from different times. Stim's \emph{detector error model} records each
error mechanism's probability and the detector and logical-observable
flips that it causes~\cite{Gidney2021stim}. One \emph{shot} samples an
error pattern and decodes its detector outcomes. A \emph{logical failure}
occurs when the decoder predicts a logical observable incorrectly.
We evaluate separate \(X\)- and \(Z\)-memory experiments and define \(\LER=p_X+p_Z\), where \(p_X\) and \(p_Z\) are their logical failure probabilities.

\paragraph{Decoders}
\label{sec:decoders}

Minimum-weight perfect matching (MWPM) is a widely used decoder for surface codes~\cite{Higgott2022pymatching}. A matching graph represents each detector by a vertex. Each error edge joins two detectors or joins one detector to a boundary. A physical fault can flip more than two detectors. Where a decomposition is available, the decoding model represents that fault through components involving at most two detectors~\cite{higgott2025blossom}.

Belief propagation with ordered-statistics decoding (BP-OSD) operates on a detector-check matrix and supports errors affecting more than two detectors~\cite{Panteleev2021bposd,roffe2020bposd,Higgott2024stimbposd}.

\paragraph{Reinforcement Learning and PPO}
\label{sec:rl-background}

Reinforcement learning trains a policy to choose actions that increase
expected cumulative reward. The environment is a Markov decision process
with states, actions, transitions, and rewards.

The actor $\pi_\theta(a\mid s)$ assigns a probability to each action $a$ in state $s$, where $\theta$ denotes the trainable parameters.
The critic $V_\theta(s)$ predicts the discounted future reward.
An \emph{advantage} estimates how much better an action's outcome is
than the critic's prediction. \emph{Generalized advantage estimation}
(GAE) combines reward and value-prediction errors across a trajectory
to assign credit to earlier decisions~\cite{Schulman2016gae}.

\emph{Proximal policy optimization} (PPO) uses those advantages to
update the actor. Its clipped objective discourages large changes in
action probabilities during each update~\cite{Schulman2017ppo}.
The critic is trained to fit the estimated returns.

\paragraph{Monte Carlo estimation.}
For a memory circuit with logical failure probability $p_L$, direct
Monte Carlo with $N$ shots estimates $p_L$ with variance
$p_L(1-p_L)/N$. For small $p_L$, achieving relative standard error
$\epsilon$ therefore requires approximately
$
N \approx 1/{(p_L\epsilon^2)}.
$
Thus, direct Monte Carlo becomes increasingly expensive as logical
failures become rare~\cite{Beverland2025failfast}.

\section{Our Reinforcement Learning Approach}
\label{sec:method}

Our approach learns which CNOT orders lead to low logical error rates.
The actor proposes orders one check at a time. We evaluate each complete
circuit and use its LER to improve the policy with
PPO~\cite{Schulman2017ppo}. After training, we select one fixed schedule
for syndrome extraction. Figure~\ref{fig:agent} illustrates a gate-order
decision for the Steane code. Figure~\ref{fig:ppo-workflow-in-sysname}
shows the training loop. Figure~\ref{fig:scheduling-workflow} connects
training to candidate re-evaluation and local improvement.

\subsection{Markov Decision Process Formulation}
\label{sec:mdp}

\begin{figure}[t]
\centering
\begin{tikzpicture}[x=1cm,y=1cm,font=\small,
  panel/.style={draw=black!45,rounded corners=2pt,line width=.5pt,fill=white},
  flow/.style={-{Latex[length=2mm]},line width=.65pt,draw=black!80},
  cell/.style={draw=black!25,line width=.3pt,fill=white}]
\draw[panel] (0,0) rectangle (8.24,-3.63);
\draw[black!25] (4.09,-.15) -- (4.09,-2.73);
\draw[black!25] (.15,-2.88) -- (8.09,-2.88);
\node at (1.97,-.28) {Chosen order for $X_A$};
\node[text=fsBlue] at (1.97,-.63) {$q_3,\ q_4,\ q_1,\ q_2$};
\node[anchor=east] at (.70,-1.70) {$H_A$};
\foreach \q in {1,...,7}{
  \node at ({1.36+.31*(\q-1)},-.99) {$\q$};}
\foreach \r in {1,...,4}{
  \node at (1.00,{-1.00-.30*\r}) {$\r$};
  \foreach \q in {1,...,7}{
    \draw[cell] ({1.215+.31*(\q-1)},{-.855-.30*\r}) rectangle ++(.29,-.29);}}
\foreach \r/\q in {1/3,2/4,3/1,4/2}{
  \fill[fsBlue!75] ({1.235+.31*(\q-1)},{-.875-.30*\r}) rectangle ++(.25,-.25);}
\node[text=black!65] at (1.97,-2.62) {Other history slices: all zero};
\node at (6.23,-.28) {Code features};
\node at (6.00,-.65) {$B$};
\node at (7.31,-.65) {$\tau$};
\node at (7.94,-.65) {$\bar w$};
\foreach \r/\lab/\typ in {0/X_A/1,1/X_B/1,2/X_C/1,3/Z_A/0,4/Z_B/0,5/Z_C/0}{
  \node[anchor=east] at (4.94,{-1.00-.29*\r}) {$\lab$};
  \foreach \q in {0,...,6}{
    \draw[cell] ({5.04+.27*\q},{-.875-.29*\r}) rectangle ++(.25,-.25);}
  \node at (7.31,{-1.00-.29*\r}) {$\typ$};
  \node at (7.94,{-1.00-.29*\r}) {$1$};}
\foreach \r/\q in {0/0,0/1,0/2,0/3,1/0,1/2,1/4,1/6,2/2,2/3,2/4,2/5,
                    3/0,3/1,3/2,3/3,4/0,4/2,4/4,4/6,5/2,5/3,5/4,5/5}{
  \fill[black!55] ({5.06+.27*\q},{-.895-.29*\r}) rectangle ++(.21,-.21);}
\node[anchor=west] at (.15,-3.32) {Current check: $X_B$};
\node[anchor=east] at (4.55,-3.32) {$c_2=$};
\foreach \j/\v in {0/0,1/1,2/0,3/0,4/0,5/0}{
  \node[draw=black!35,minimum width=.43cm,minimum height=.37cm,inner sep=0pt]
    at ({4.85+.49*\j},-3.32) {$\v$};}
\draw[fsBlue,line width=1pt] (5.12,-3.52) rectangle (5.56,-3.12);
\draw[flow] (4.12,-3.63) -- node[right] {flatten and concatenate} (4.12,-4.10);
\node[anchor=east] at (.64,-4.33) {$s_2$};
\foreach \left/\right/\lab/\shade in {.8/1.6/c_2/10,1.6/4.0/H/20,4.0/6.1/B/10,6.1/6.9/\tau/20,6.9/7.7/\bar w/10}{
  \draw[draw=black!45,fill=black!\shade] (\left,-4.12) rectangle (\right,-4.55);
  \node at ({(\left+\right)/2},-4.33) {$\lab$};}
\draw[flow] (1.30,-4.55) -- (1.30,-5.08);
\node[anchor=west,text=black!65] at (3.00,-4.81) {228 raw features};
\draw[panel] (0,-5.10) rectangle (4.90,-6.37);
\node[font=\scriptsize] at (1.30,-5.35) {128 neurons, tanh};
\node[font=\scriptsize] at (3.60,-5.35) {128 neurons, tanh};
\foreach \ya in {-5.67,-5.88,-6.09}{\foreach \yb in {-5.67,-5.88,-6.09}{
  \draw[black!18,line width=.35pt] (1.40,\ya) -- (3.50,\yb);}}
\foreach \x in {1.30,3.60}{\foreach \y in {-5.67,-5.88,-6.09}{
  \draw[draw=fsBlue,fill=white,line width=.5pt] (\x,\y) circle (.075);}}
\node[panel,minimum width=2.75cm,minimum height=1.27cm,align=center]
  (actor) at (6.87,-5.735) {actor $\pi_\theta(a_2\mid s_2)$\\24 valid orders};
\draw[flow] (4.90,-5.735) -- (actor.west);
\draw[flow] (actor.south) -- node[right] {sample} (6.87,-6.91);
\draw[panel] (0,-6.93) rectangle (8.24,-7.48);
\node[text=fsBlue] at (4.12,-7.20)
  {Selected order: $a_2=11\quad\longrightarrow\quad q_3,\ q_7,\ q_1,\ q_5$};
\end{tikzpicture}
\caption{Encoding a partial syndrome-extraction schedule.}
\Description{Decision two of an episode for the seven-qubit Steane code. The chosen order three,
  four, one, two for check X A is a four by seven one-hot matrix. The support
  matrix, binary check types, normalized weights, and current-check vector
  complete the state. After flattening, two shared 128-unit tanh layers feed
  the actor. Action eleven chooses order three, seven, one, five for X B.}
\label{fig:agent}
\end{figure}
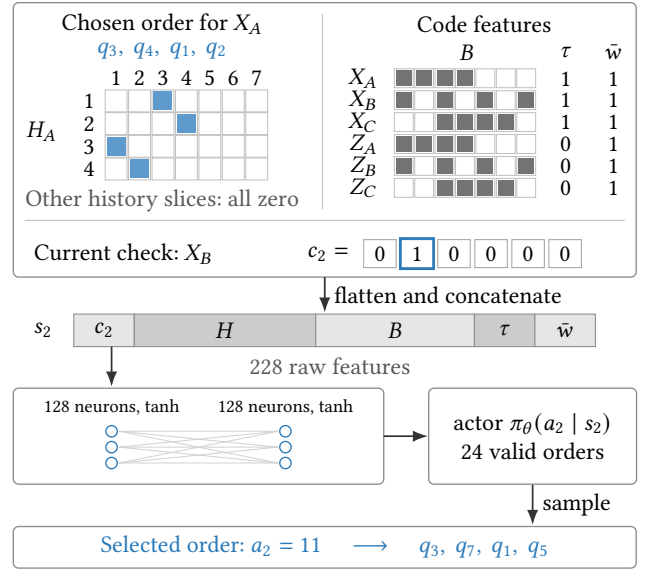

Gate orders interact across checks. We represent these dependencies
with a Markov decision process (MDP).
A state $s_i$ records the partial schedule before decision $i$.
Action $a_i$ selects the CNOT order for check $i$.
The transition records that order and advances to the next check.
For $m$ checks, an \emph{episode} makes $m$ decisions and produces a
complete schedule. The recorded states, actions, and rewards form a
\emph{trajectory}. The circuit's LER supplies the terminal reward
(Section~\ref{sec:rl-reward}).

The $\code{7}{1}{3}$ Steane code serves as a running example.
For either Pauli type, checks $A$, $B$, and $C$ have supports
$\{1,2,3,4\}$, $\{1,3,5,7\}$, and $\{3,4,5,6\}$.
We visit the three $X$ checks followed by the three $Z$ checks.
Figure~\ref{fig:agent} shows the second decision: the actor has chosen
$(3,4,1,2)$ for $X_A$ and now chooses an order for $X_B$.
All displayed indices are one-based.


\paragraph{Encoding a partial schedule.}
The state combines earlier choices with the code structure:
\begin{equation}
s_i=\bigl[c_i\,\Vert\,H\,\Vert\,B\,\Vert\,\tau\,\Vert\,\bar w\bigr].
\label{eq:state}
\end{equation}
Here $\Vert$ denotes concatenation after flattening the arrays.
The vector $c_i\in\{0,1\}^m$ marks the current check.
The history $H\in\{0,1\}^{m\times w_{\max}\times n}$ has
$H_{j,\ell,q}=1$ when slot $\ell$ of check $j$ visits qubit $q$.
Undecided checks and unused slots are zero.
The fixed features are the support matrix $B_{j,q}$, check type
$\tau_j$ (one for $X$, zero for $Z$), and normalized weight
$\bar w_j=w_j/w_{\max}$.
For Steane, these arrays give 228 inputs.
In Figure~\ref{fig:agent}, $H_A$ records $3,4,1,2$, $c_2$ marks $X_B$,
and $B$ identifies shared qubits 1 and 3.

\paragraph{Choosing an action.}
The actor chooses from a fixed table $\mathcal A_i$ of CNOT orders.
Checks with $w_i!\leq32$ use all permutations. Larger checks use 32
orders drawn from rotations, reversals, and random permutations of the
sorted support. The cap lets training revisit actions within a fixed
budget. Local improvement can add further orders
(Section~\ref{sec:selection}). Appendix~\ref{app:training-details}
gives the reproducible table construction.

Two shared, fully connected layers of 128 $\tanh$ units feed the actor
and critic. The actor masks invalid table entries before sampling.
The critic predicts $V_\theta(s_i)$, where $\theta$ contains all network
weights and biases. Recording the chosen order in $H$ and advancing
$c_i$ produces the next state.

\paragraph{Turning the orders into a circuit.}
A deterministic scheduler assigns the chosen gates to time steps.
All $X$-check CNOTs form an \emph{$X$ block}, which finishes before
the \emph{$Z$ block} begins. Gates from several checks can overlap
within each block.

Within a block, the scheduler visits CNOT positions in order and checks
in their fixed order. Each gate takes the earliest step after its check's
preceding gate at which both qubits are free.
In the Steane example, $X_A$ and $X_B$ both begin on $q_3$.
The scheduler places $X_A$'s first gate at $t_1$ and $X_B$'s first gate at $t_2$.
These placement rules determine parallelism, idle periods, and circuit depth.

\subsection{Objective and Reward}
\label{sec:rl-reward}

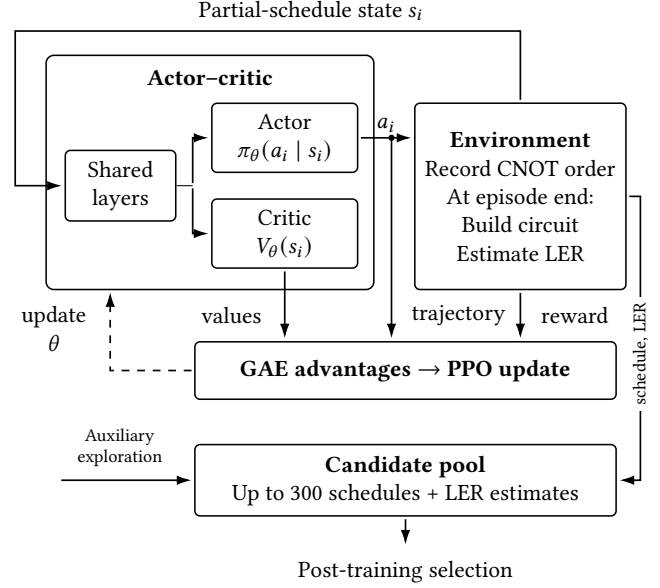
\begin{figure}[t]
\centering
\resizebox{\columnwidth}{!}{%
\begin{tikzpicture}[x=1cm,y=1cm,font=\small,
  block/.style={draw=black,rounded corners=2pt,fill=white,
    line width=.6pt,align=center,inner sep=4pt},
  flow/.style={-{Latex[length=2mm,width=1.3mm]},draw=black,
    line width=.65pt,shorten <=1pt,shorten >=1pt},
  update/.style={flow,dashed}]
\draw[black,rounded corners=3pt,line width=.6pt]
  (.48,-.15) rectangle (4.93,-3.33);
\node[font=\small\bfseries] at (2.705,-.45) {Actor--critic};
\node[block,minimum width=1.50cm,minimum height=.82cm]
  (shared) at (1.49,-1.920) {Shared\\layers};
\node[block,minimum width=1.98cm,minimum height=.80cm]
  (actor) at (3.72,-1.27) {Actor\\$\pi_\theta(a_i\mid s_i)$};
\node[block,minimum width=1.98cm,minimum height=.80cm]
  (critic) at (3.72,-2.57) {Critic\\$V_\theta(s_i)$};
\node[block,minimum width=2.60cm,minimum height=2.53cm]
  (env) at (6.92,-2.065)
  {\textbf{Environment}\\Record CNOT order\\At episode end:\\Build circuit\\Estimate LER};

\draw[flow] (env.north) -- (6.92,.18) -- (.06,.18)
  -- (.06,-1.920) -- (shared.west);
\node at (4.08,.46) {Partial-schedule state $s_i$};
\coordinate (split) at (2.47,-1.920);
\draw (shared.east) -- (split);
\draw[flow] (split) |- (actor.west);
\draw[flow] (split) |- (critic.west);
\draw[flow] (actor.east) --
  node[above,inner sep=2pt] {$a_i$} (env.west |- actor.east);

\node[block,minimum width=5.72cm,minimum height=.78cm,
      font=\small\bfseries] (train) at (5.36,-4.43) {GAE advantages $\to$ PPO update};
\draw[flow] (critic.south) -- (critic.south |- train.north);
\node[anchor=east] at (3.53,-3.68) {values};
\coordinate (rollout) at (5.17,-1.27);
\fill (rollout) circle (1.2pt);
\draw[flow] (rollout) -- (rollout |- train.north);
\node[anchor=west] at (5.32,-3.68) {trajectory};
\draw[flow] (env.south) -- (env.south |- train.north);
\node[anchor=west] at (7.08,-3.68) {reward};
\draw[update] (train.west) -- (1.35,-4.43) -- (1.35,-3.33);
\node[anchor=east,align=center] at (1.15,-3.86) {update\\$\theta$};
\node[block,minimum width=5.72cm,minimum height=.90cm]
  (pool) at (5.36,-5.92)
  {\textbf{Candidate pool}\\Up to 300 schedules + LER estimates};
\draw[flow] (env.east) -- (8.55,-2.065) -- (8.55,-5.92) -- (pool.east);
\node[font=\scriptsize,rotate=90,fill=white,inner sep=2pt]
  at (8.55,-4.20) {schedule, LER};
\draw[flow] (.65,-5.92) -- (pool.west);
\node[font=\scriptsize,align=center,fill=white,inner sep=1pt]
  at (1.48,-5.47) {Auxiliary\\exploration};
\draw[flow] (pool.south) -- (5.36,-6.83);
\node[anchor=north,font=\small] at (5.36,-6.89) {Post-training selection};
\end{tikzpicture}}
\caption{Actor--critic PPO training loop in \sysname{}.}

\label{fig:ppo-workflow-in-sysname}

\Description{The interaction row contains shared layers feeding actor and
  critic heads beside the scheduling environment. A state loop returns from
  the environment to the shared layers. The actor sends an action to the
  environment. Separate paths carry the recorded trajectory, critic values, and
  terminal reward to GAE and PPO in a second row. A dashed arrow updates the
  entire actor-critic group. Completed schedules and their logical error
  estimates also enter a candidate pool of up to 300 entries. Auxiliary
  exploration contributes to this pool. The pool feeds post-training selection.}
\end{figure}

The reward should favor reliable circuits while encouraging the policy
to try less-used orders. We first evaluate a complete schedule $\sigma$
with the $Z$- and $X$-memory circuits from
Section~\ref{sec:score-background}. Direct Monte Carlo samples error
patterns and decodes each shot. With $f_b$ failures in $N_b$ shots for
$b\in\{Z,X\}$, the estimated score is
\begin{equation}
\widehat{\LER}(\sigma)=\frac{f_Z}{N_Z}+\frac{f_X}{N_X}.
\label{eq:mc-ler}
\end{equation}
For a code with several logical qubits, any incorrect tracked observable
counts as a failure. Code, noise, memory duration, and decoder settings
stay fixed during search.

We convert LER into a reward on a logarithmic scale and add an
exploration bonus $B(\sigma)$:
\begin{equation}
\begin{aligned}
r_{\mathrm{LER}}(\sigma)
&=-\log_{10}\max\bigl(\widehat{\LER}(\sigma),\varepsilon_L\bigr),\\
r_i&=\begin{cases}
0, & i<m,\\[2pt]
r_{\mathrm{LER}}(\sigma)+B(\sigma), & i=m.
\end{cases}
\end{aligned}
\label{eq:reward}
\end{equation}

The floor $\varepsilon_L=10^{-9}$ keeps the logarithm finite.
Above this floor, halving LER adds $\log_{10}2$ to the LER reward.
Training rewards update PPO. Independent assessment measures the selected
schedule's LER after search.

The exploration bonus rewards orders that the search has used less often:
\begin{equation}
B(\sigma)=\sum_{j=1}^{m}\frac{\beta}{\sqrt{C_j(a_j)}}.
\end{equation}
Here $C_j(a)$ is one plus the number of earlier selections of action $a$
for check $j$ in the current run, read before the current selection.
Policy and exploration candidates both update these counts.
We use $\beta=0.02$. The inverse-square-root form follows count-based
exploration~\cite{Strehl2008mbie,Bellemare2016count}.
The bonus guides training. Final selection ranks schedules by LER.

\subsection{Learning Algorithm}
\label{sec:ppo}

\input{tex/fig-scheduling-workflow}

Training turns feedback on complete circuits into better gate-order
choices. Each batch has three steps: collect schedules with the current
policy, evaluate their LERs, and update the actor and critic.
Figure~\ref{fig:ppo-workflow-in-sysname} shows this loop.
Each independent run starts with new network parameters and visit counts.
The run also maintains a \emph{candidate pool}: a list of up to 300
distinct schedules and their estimated LERs, saved for final selection.
After each batch, we retain the entries with the lowest LER estimates.
Training continues within preset batch, time, and optional shot limits.
Appendix~\ref{app:training-details} explains this budget check in
Figure~\ref{fig:scheduling-workflow}.

\paragraph{Assigning credit to earlier choices.}
The final circuit reveals the reward only after all checks have been
scheduled. We store the states, chosen actions, action probabilities,
and critic predictions during collection. Once the terminal reward is
available, generalized advantage estimation (GAE) works backward through
the stored trajectory~\cite{Schulman2016gae}.
The actor and critic stay fixed until collection and evaluation finish.

For the batch's critic $V_{\mathrm{old}}$, we compute
\begin{equation}
\begin{aligned}
\delta_i&=r_i+\gamma V_{\mathrm{old}}(s_{i+1})-V_{\mathrm{old}}(s_i),\\
\widehat A_i&=\delta_i+\gamma\lambda\widehat A_{i+1}.
\end{aligned}
\label{eq:gae}
\end{equation}
The boundary values are $V_{\mathrm{old}}(s_{m+1})=\widehat A_{m+1}=0$.
We use $\gamma=1$ and $\lambda=0.95$.
The recursion passes later prediction errors to earlier gate-order
choices.
The critic target is $y_i=V_{\mathrm{old}}(s_i)+\widehat A_i$.
We center and scale the advantages across the batch to obtain
$\widetilde A_i$.

\paragraph{Updating the policy.}
PPO uses these advantages to adjust the probabilities of the recorded
actions. Let $\pi_{\mathrm{old}}$ be the actor that collected the batch.
The probability ratio and its clipped version are
\begin{equation}
\rho_i=\frac{\pi_\theta(a_i\mid s_i)}{\pi_{\mathrm{old}}(a_i\mid s_i)},
\qquad \bar\rho_i=\operatorname{clip}(\rho_i,0.8,1.2).
\end{equation}
PPO maximizes the clipped objective~\cite{Schulman2017ppo}:
\begin{equation}
L_{\mathrm{clip}}=
\left\langle\min\bigl(\rho_i\widetilde A_i,
\bar\rho_i\widetilde A_i\bigr)\right\rangle.
\label{eq:ppo-clip}
\end{equation}
Angle brackets denote the batch mean over policy decisions.
A positive advantage favors increasing the chosen action's probability.
A negative advantage favors a decrease.
Clipping limits the credited gain from large changes
(Appendix~\ref{app:training-details}).

We make four update passes over the full batch before collecting new
trajectories. Each pass minimizes
\begin{equation}
-L_{\mathrm{clip}}
+0.5\left\langle\bigl(V_\theta(s_i)-y_i\bigr)^2\right\rangle
-0.03\left\langle\mathcal H(\pi_\theta(\cdot\mid s_i))\right\rangle.
\label{eq:ppo-loss}
\end{equation}
where \(\mathcal H\) denotes policy entropy.
The squared-error term trains the critic, and the entropy term encourages
action diversity. We use Adam with learning rate $3\times10^{-4}$ and
clip the gradient norm to $0.5$.

\paragraph{Coordinated exploration.}
Shared gate-order patterns help explore coordinated choices across
checks. Auxiliary candidates draw either one shared table index per
group of checks with the same Pauli type and weight, or independent
uniform indices. These candidates enter the saved pool and update visit
counts. PPO uses actor-generated trajectories.
Appendix~\ref{app:exploration} gives the settings and an example.
Section~\ref{sec:ablation} measures the effect.

\subsection{Post-Training Selection}
\label{sec:selection}

Deployment requires one fixed schedule. We merge and deduplicate the
training pools, then rank a shortlist in two passes with fresh samples.
The second pass gives the survivors larger budgets.

Optional local improvement begins separately from the best re-evaluated
schedules. Each starting schedule initializes one local search.
A round tests changes to one check's CNOT order.
We accept a change when its estimated LER
improvement exceeds one combined standard error. Close positive
differences receive more samples. An accepted schedule receives a fresh
LER estimate before the next round. A rejected change, the round limit,
or exhausted untested changes ends that local search. We return the
refined schedule with the lowest latest estimate. Disabling local improvement
returns the best re-evaluated candidate.

Figure~\ref{fig:scheduling-workflow} shows the fixed output.
Appendix~\ref{app:selection} gives the budgets, acceptance rule, and
stopping conditions. Selection counts toward compilation cost.
Independent assessment follows selection (Section~\ref{sec:sig}).

\section{Schedule Evaluation}
\label{sec:amplification}

Direct Monte Carlo becomes expensive as schedules improve and logical
failures become rare (Section~\ref{sec:score-background}).
For example, at $p_L=10^{-4}$, achieving relative standard error $0.1$
requires about $10^6$ shots per memory circuit.
Moreover, a direct evaluation observes zero failures with probability
$(1-p_L)^N$. When $Np_L\ll1$, many schedules therefore receive the same
floor-based reward.

We use importance sampling to obtain more informative LER estimates with
fewer shots. We amplify fault probabilities to observe failures more
frequently, then correct for the changed sampling distribution.

\subsection{Overview}

\input{tex/fig-noise-training-loop}

Figure~\ref{fig:noise-training-loop} illustrates our schedule evaluator.
To evaluate a schedule $\sigma$ under noise model $P$, we sample under
an amplified distribution $Q$. Weighting each failed sample by the
likelihood ratio $P/Q$ estimates the LER under
$P$~\cite{Owen2013mc,Frank2008rare}.
The resulting LER estimate provides the reward signal, which PPO uses to update the trainable parameters $\theta$ of the agent's policy.
The candidate circuit and target-noise decoder stay fixed while the
sampling distribution changes. Section~\ref{sec:learning-curves}
compares training-shot counts under direct Monte Carlo and importance
sampling with matched budgets.

\subsection{Importance-Sampling Estimator}
\label{sec:estimator}

We use Stim to construct a detector error model for schedule $\sigma$.
Let \(\mathcal E\) contain the independent error mechanisms in this detector error model.
Each mechanism $e\in\mathcal E$ occurs with probability $0<p_e<1$ and
flips fixed detectors and logical observables.
For amplification factor $k>1$, the sampling probability is

$$
q_e(k)=\min\{kp_e,0.45\}.
$$
Amplification makes failures more frequent. The implementation caps
each probability at 0.45 to limit the change in the error distribution.
Unbiasedness follows from the likelihood correction below.
At $k=1$, we use $q_e=p_e$ exactly. Mechanisms with probability zero or one
retain their original probabilities for every $k$.

Let $P$ denote the target error probability distribution with
probabilities $p_e$, and $Q_k$ the amplified error probability distribution
with probabilities $q_e(k)$.
An error pattern $x\in\{0,1\}^{|\mathcal E|}$ records which mechanisms
occur. We keep the decoder built for $P$ fixed while sampling from $Q_k$.

Let \(F_\sigma(x)\in\{0,1\}\) indicate whether pattern \(x\) causes a logical failure.
For a failure indicator $F_\sigma(x)$ and $N$ independent samples, the
importance-sampling estimator is
\begin{equation}
\begin{aligned}
W_k(x)
&=\frac{P(x)}{Q_k(x)}
=\prod_{e\in\mathcal E}
\left(\frac{p_e}{q_e}\right)^{x_e}
\left(\frac{1-p_e}{1-q_e}\right)^{1-x_e},
\\
\widehat p_{L,k}(\sigma)
&=\frac{1}{N}\sum_{h=1}^{N}
W_k(x^{(h)})F_\sigma(x^{(h)}).
\end{aligned}
\label{eq:is}
\end{equation}
We compute weights~\cite{Owen2013mc} in log space, including occurring
and absent mechanisms. Every error pattern possible under $P$ also has
positive probability under $Q_k$. At $k=1$, all weights equal one and the
estimator is direct MC.

\begin{theorem}[Fixed-budget unbiasedness]
\label{thm:unbiased}
For fixed $\sigma$, decoder, $k$, and $N\ge1$, if $Q_k(x)>0$ whenever
$P(x)>0$, then
$$
\mathbb E_{Q_k}[\widehat p_{L,k}(\sigma)]=p_L(\sigma).
$$
\end{theorem}

\begin{proof}
For one sample,
$$
\sum_x Q_k(x)W_k(x)F_\sigma(x)
=\sum_xP(x)F_\sigma(x)
=p_L(\sigma).
$$
Linearity gives the result for the sample mean.
\end{proof}

Summing the separate $Z$- and $X$-memory estimates gives
$\widehat{\LER}_k$.

Substituting the weighted estimate into Eq.~\eqref{eq:reward} gives
the training reward:
\begin{equation}
r_m^{(k)}
=-\log_{10}\max\bigl(\widehat{\LER}_k(\sigma),\varepsilon_L\bigr)
+B(\sigma).
\label{eq:weighted-reward}
\end{equation}

Importance sampling supplies LER estimates to all three stages in
Figure~\ref{fig:scheduling-workflow}, using the RL algorithm from
Section~\ref{sec:method}. Local improvement combines the weighted
uncertainty in Appendix~\ref{app:weighted-uncertainty} with the acceptance
rule in Appendix~\ref{app:selection}.

Theorem~\ref{thm:unbiased} applies to the fixed-budget LER estimator.
The nonlinear logarithm can nevertheless bias finite-sample rewards,
and adaptive stopping below introduces an additional source of bias.
We therefore assess final schedules independently with target-noise MC
(Section~\ref{sec:sig}).

\subsection{Choosing the Amplification Factor}
\label{sec:calibration}

Three pilots calibrate $k$ to balance failure frequency, weight
variance, and decoding cost. The median choice remains fixed throughout
search; Appendix~\ref{app:calibration} provides the full procedure and
Figure~\ref{fig:evaluation-control} its workflow.

\subsection{Allocating Shots per Schedule}
\label{sec:comparison}

Each candidate needs enough failure information to guide the search.
The effective sample size (ESS) of failure contributions supplies a
stopping rule that accounts for unequal importance weights.

We pool equal shot counts from the two memory circuits.
Let $Y_i$ be the importance weight of a failed shot and zero otherwise.
Each normalized contribution $\alpha_i=Y_i/\sum_jY_j$ is that shot's
share of the total weighted failures. Squaring these shares emphasizes
large contributors. The reciprocal of their sum gives the effective
failure count~\cite[Section~9.3]{Owen2013mc}:
\begin{equation}
\ESS_{\mathrm{fail}}
=\frac{1}{\sum_i\alpha_i^2}
=\frac{\bigl(\sum_iY_i\bigr)^2}{\sum_iY_i^2}.
\label{eq:ess}
\end{equation}
We set the count to zero until a failure contributes positive weight.
For equally weighted failures, the count is their number. At $k=1$,
it therefore equals the observed failure count.

\paragraph{Stopping an evaluation.}
We sample both memory circuits in batches until
$\ESS_{\mathrm{fail}}$ reaches target $T$ or shots per circuit reach
$\Nmax$.
Training uses $T=30$. Candidate re-evaluation raises the target to 200
and then 500 as the pool narrows.
Local improvement similarly spends larger budgets on promising changes
and small estimated improvements, using Eq.~\eqref{eq:accept} with
weighted LER and the uncertainty estimate in Eq.~\eqref{eq:is-uncertainty}.
This margin is a search heuristic under adaptive stopping.
Thus $k$ controls the error distribution, while $T$ and $\Nmax$ control
the information target and maximum cost of each evaluation.

With a fixed shot budget, the estimator's variance falls as $1/N$
and Theorem~\ref{thm:unbiased} applies. Adaptive stopping can bias the
estimate. Independent target-noise assessment measures the selected
schedule's LER (Section~\ref{sec:sig}).

\section{Evaluation}
\label{sec:eval}
Our evaluation addresses three questions.
\begin{description}[style=unboxed,leftmargin=0pt,labelindent=0pt,itemsep=1pt,topsep=3pt]
  \item[\textbf{RQ1---Effectiveness:}] How much does LER improve over prior methods?
  \item[\textbf{RQ2---Scalability:}] Does \sysname{} scale as code distance grows and physical noise falls?
  \item[\textbf{RQ3---Ablation study:}] Which components contribute to the final schedule quality?
\end{description}

\begin{figure*}[t]
\centering
\begingroup
\fscircuitstyle
\begin{tikzpicture}[x=1cm,y=1cm,font=\footnotesize,
 wire/.style={draw=fsCircuitInk!65,line width=.3pt},
 tick/.style={draw=black!15,line width=.3pt}]

\renewcommand{\fscnot}[4]{%
  \draw[#4,line width=.5pt] ({#1},{#2}) -- ({#1},{#3});
  \fill[#4] ({#1},{#3}) circle (1.3pt);
  \coordinate (fsgatecenter) at ({#1},{#2});
  \draw[#4,fill=white,line width=.5pt] (fsgatecenter) circle (2.15pt);
  \draw[#4,line width=.5pt] ([xshift=-2.15pt]fsgatecenter) -- ([xshift=2.15pt]fsgatecenter);
  \draw[#4,line width=.5pt] ([yshift=-2.15pt]fsgatecenter) -- ([yshift=2.15pt]fsgatecenter);}
\newcommand{\fscircuitframe}[1]{%
  \pgfmathsetmacro{\fsmeasure}{1.93+1.38*#1}%
  \foreach \t in {1,...,#1}{
    \pgfmathsetmacro{\fsleft}{1.62+1.38*(\t-1)}%
    \ifodd\t \fill[black!3] (\fsleft,.14) rectangle ++(1.38,-4.30);\fi
    \draw[tick] (\fsleft,.14) -- ++(0,-4.30);
    \node at ({2.31+1.38*(\t-1)},.35) {$t_{\t}$};}
  \draw[tick] ({1.62+1.38*#1},.14) -- ++(0,-4.30);
  \foreach \r in {1,...,7}{
    \node[anchor=east] at (1.08,{-.32*(\r-1)}) {$q_{\r}$};
    \draw[wire] (1.26,{-.32*(\r-1)}) -- (\fsmeasure,{-.32*(\r-1)});}
  \foreach \r/\lab/\prep/\meas in {0/X_A/+/X,1/X_B/+/X,2/X_C/+/X,3/Z_A/0/Z,4/Z_B/0/Z,5/Z_C/0/Z}{
    \node[anchor=east] at (.65,{-2.46-.32*\r}) {$\lab$};
    \node at (1.02,{-2.46-.32*\r}) {$\ket{\prep}$};
    \draw[wire] (1.26,{-2.46-.32*\r}) -- (\fsmeasure,{-2.46-.32*\r});
    \node[fill=white,inner sep=1pt] at (\fsmeasure,{-2.46-.32*\r}) {$M_{\meas}$};}
  \draw[black!20] (.20,-2.19) -- (\fsmeasure,-2.19);}
\newcommand{\fsxgate}[4]{%
  \fscnot{1.85+.46*(3*(#1-1)+#2)}{-.32*(#3-1)}{-2.46-.32*#2}{#4}}
\newcommand{\fszgate}[4]{%
  \fscnot{1.85+.46*(3*(#1-1)+#2)}{-3.42-.32*#2}{-.32*(#3-1)}{#4}}
\newcommand{\fshooktrace}[2]{%
  \draw[#1,line width=1.2pt] (4.46,-3.10) -- (6.91,-3.10);
  \node[draw=#1,fill=white,inner sep=1pt,
        line width=.9pt,text=#1] at (4.66,-3.10) {$X$};
  \draw[#1,line width=1.2pt] (5.53,{-.32*(#2-1)}) -- (7.34,{-.32*(#2-1)});
  \draw[#1,line width=1.2pt] (6.91,-.64) -- (7.34,-.64);
  \node[fill=white,inner sep=1pt,text=#1] at (7.45,{-.32*(#2-1)}) {$X$};
  \node[fill=white,inner sep=1pt,text=#1] at (7.45,-.64) {$X$};}

\node[anchor=west,font=\small\bfseries] at (0,.52)
 {(a) First policy candidate: CNOT depth 10};
\node[anchor=west,font=\small\bfseries] at (8.60,.52)
 {(b) Delivered schedule: CNOT depth 9};
\begin{scope}[shift={(0,-.29)},x=.50cm,y=.86cm,font=\scriptsize]

  \fscircuitframe{10}
  \fsxgate{1}{0}{3}{black!80}
  \fsxgate{2}{0}{4}{black!80}
  \fsxgate{3}{0}{1}{black!80}
  \fsxgate{4}{0}{2}{black!80}
  \fsxgate{2}{1}{3}{black!80}
  \fsxgate{3}{1}{7}{black!80}
  \fsxgate{4}{1}{1}{black!80}
  \fsxgate{5}{1}{5}{black!80}
  \fsxgate{1}{2}{4}{black!80}
  \fsxgate{2}{2}{6}{black!80}
  \fsxgate{3}{2}{5}{fsRed}
  \fsxgate{4}{2}{3}{fsRed}
  \fszgate{6}{0}{1}{black!80}
  \fszgate{7}{0}{4}{black!80}
  \fszgate{8}{0}{2}{black!80}
  \fszgate{9}{0}{3}{black!80}
  \fszgate{7}{1}{1}{black!80}
  \fszgate{8}{1}{3}{black!80}
  \fszgate{9}{1}{7}{black!80}
  \fszgate{10}{1}{5}{black!80}
  \fszgate{6}{2}{4}{black!80}
  \fszgate{7}{2}{3}{black!80}
  \fszgate{8}{2}{5}{black!80}
  \fszgate{9}{2}{6}{black!80}
  \fshooktrace{fsRed}{5}
\end{scope}
\begin{scope}[shift={(8.60,-.29)},x=.50cm,y=.86cm,font=\scriptsize]

  \fscircuitframe{9}
  \fsxgate{1}{0}{3}{black!80}
  \fsxgate{2}{0}{4}{black!80}
  \fsxgate{3}{0}{2}{black!80}
  \fsxgate{4}{0}{1}{black!80}
  \fsxgate{1}{1}{1}{black!80}
  \fsxgate{2}{1}{3}{black!80}
  \fsxgate{3}{1}{7}{black!80}
  \fsxgate{4}{1}{5}{black!80}
  \fsxgate{1}{2}{5}{black!80}
  \fsxgate{2}{2}{6}{black!80}
  \fsxgate{3}{2}{4}{fsBlue}
  \fsxgate{4}{2}{3}{fsBlue}
  \fszgate{5}{0}{1}{black!80}
  \fszgate{6}{0}{2}{black!80}
  \fszgate{7}{0}{4}{black!80}
  \fszgate{8}{0}{3}{black!80}
  \fszgate{5}{1}{5}{black!80}
  \fszgate{6}{1}{7}{black!80}
  \fszgate{7}{1}{3}{black!80}
  \fszgate{8}{1}{1}{black!80}
  \fszgate{6}{2}{5}{black!80}
  \fszgate{7}{2}{6}{black!80}
  \fszgate{8}{2}{4}{black!80}
  \fszgate{9}{2}{3}{black!80}
  \fshooktrace{fsBlue}{4}
\end{scope}
\draw[fsRed,line width=.6pt] (0,-4.06) -- (8.15,-4.06);
\node at (4.075,-4.32) {Fault after $X_C$ CNOT 2: $E=X_3X_5$};
\node at (4.075,-4.70) {Syndrome $(1,0,0)$, BP-OSD recovery $R\sim X_2$};
\node[text=fsRed,font=\footnotesize\bfseries] at (4.075,-5.10)
 {$ER=X_2X_3X_5=\overline X$: logical failure};
\draw[fsBlue,line width=.6pt] (8.60,-4.06) -- (16.85,-4.06);
\node at (12.725,-4.32) {Same fault and slot: $E=X_3X_4$};
\node at (12.725,-4.70) {Syndrome $(0,1,0)$, BP-OSD recovery $R\sim X_1X_2$};
\node[text=fsBlue,font=\footnotesize\bfseries] at (12.725,-5.10)
 {$ER=X_1X_2X_3X_4=X_A$: harmless stabilizer};
\end{tikzpicture}
\endgroup
\caption{Two Steane schedules produce different decoded outcomes from
  the same ancilla fault under Brisbane noise and BP-OSD.
  Here $\sim$ denotes equivalence up to a stabilizer.}
\Description{Two recorded circuit layouts have seven data wires and six
  ancillas. The first candidate has depth ten and the selected circuit has
  depth nine. The same
  X fault on the X C ancilla after its second CNOT produces hook X3 X5 in
  the first circuit. Syndrome 100 leads to a recovery equivalent to X2,
  leaving logical X2 X3 X5. The delivered circuit produces hook X3 X4.
  Syndrome 010 leads to a recovery equivalent to X1 X2, leaving the
  harmless stabilizer X1 X2 X3 X4. Rust and blue highlight these paths.}
\label{fig:full-pipeline}
\end{figure*}
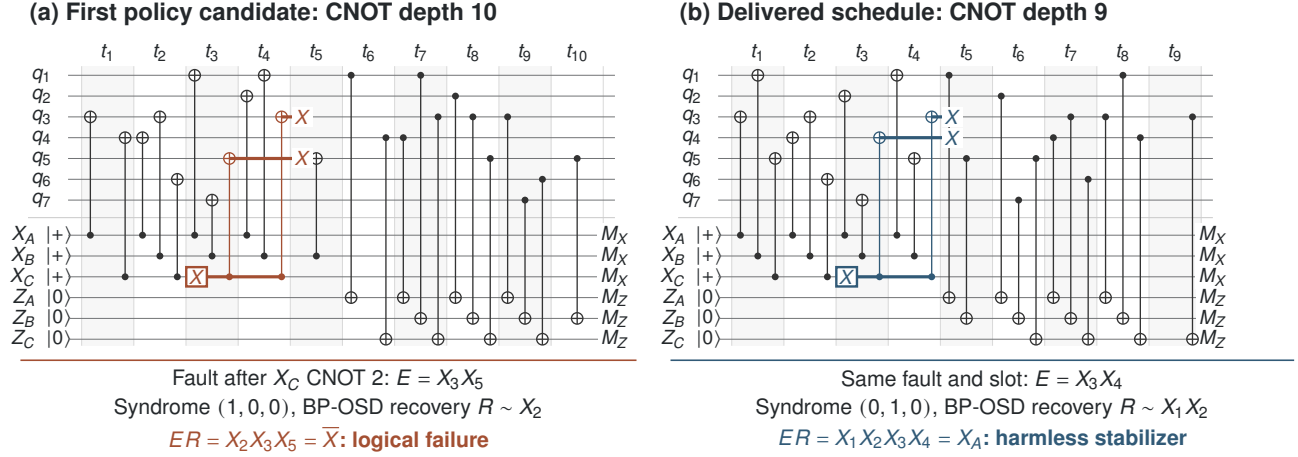

\begin{table}[tbp]
\centering\small
\setlength{\tabcolsep}{2.5pt}
\caption{Final LERs and reductions versus the released \AS{} and \PH{}
baselines. LERs have three significant figures. Reductions use unrounded
rates: $100(1-\LER_{\mathrm{FS}}/\LER_{\mathrm{base}})\%$.
Suite averages weight circuits equally. Bold $\downarrow$ marks an
estimated reduction, not a confidence claim.}
\label{tab:rq1-compact}
\begin{tabular}{@{}llrrr@{}}
\toprule
Circuit & $\llbracket n,k,d\rrbracket$ & \shortstack[r]{Baseline\\LER} & \shortstack[r]{Our\\LER} & \shortstack[r]{Reduction\\(\%)}\\
\midrule
\multicolumn{5}{c}{\textbf{Versus \AS{}: 11 circuits, \meanRed\% avg. decrease}}\\
\midrule
Hex. col. & \code{7}{1}{3} & $1.02\!\times\!10^{-3}$ & $7.57\!\times\!10^{-4}$ & $\boldsymbol{25.6\downarrow}$\\
Sq.-oct. col. & \code{7}{1}{3} & $9.34\!\times\!10^{-4}$ & $7.37\!\times\!10^{-4}$ & $\boldsymbol{21.1\downarrow}$\\
Surface & \code{9}{1}{3} & $5.00\!\times\!10^{-4}$ & $4.40\!\times\!10^{-4}$ & $\boldsymbol{12.1\downarrow}$\\
Hex. col. & \code{19}{1}{5} & $9.97\!\times\!10^{-4}$ & $5.04\!\times\!10^{-4}$ & $\boldsymbol{49.4\downarrow}$\\
Sq.-oct. col. & \code{17}{1}{5} & $1.13\!\times\!10^{-3}$ & $6.90\!\times\!10^{-4}$ & $\boldsymbol{39.2\downarrow}$\\
Hyp. col. & \code{24}{8}{4} & $4.25\!\times\!10^{-2}$ & $3.13\!\times\!10^{-2}$ & $\boldsymbol{26.4\downarrow}$\\
Surface & \code{25}{1}{5} & $6.80\!\times\!10^{-5}$ & $3.53\!\times\!10^{-5}$ & $\boldsymbol{48.0\downarrow}$\\
Defect surf. & \code{25}{2}{5} & $1.88\!\times\!10^{-3}$ & $1.63\!\times\!10^{-3}$ & $\boldsymbol{13.4\downarrow}$\\
Hyp. surf. & \code{30}{8}{3} & $1.75\!\times\!10^{-2}$ & $1.67\!\times\!10^{-2}$ & $\boldsymbol{4.4\downarrow}$\\
Rect. surf. & \code{45}{1}{5} & $1.21\!\times\!10^{-4}$ & $7.92\!\times\!10^{-5}$ & $\boldsymbol{34.4\downarrow}$\\
Hyp. surf. & \code{60}{18}{3} & $4.55\!\times\!10^{-2}$ & $4.03\!\times\!10^{-2}$ & $\boldsymbol{11.5\downarrow}$\\
\midrule
\multicolumn{5}{c}{\textbf{Versus \PH{}: 8 circuits, 71.7\% avg. decrease}}\\
\midrule
Surface & \code{9}{1}{3} & $2.30\!\times\!10^{-2}$ & $2.59\!\times\!10^{-3}$ & $\boldsymbol{88.7\downarrow}$\\
Surface & \code{25}{1}{5} & $6.61\!\times\!10^{-3}$ & $9.10\!\times\!10^{-4}$ & $\boldsymbol{86.2\downarrow}$\\
Surface & \code{49}{1}{7} & $1.65\!\times\!10^{-2}$ & $1.74\!\times\!10^{-4}$ & $\boldsymbol{98.9\downarrow}$\\
Surface & \code{81}{1}{9} & $6.60\!\times\!10^{-3}$ & $2.48\!\times\!10^{-4}$ & $\boldsymbol{96.2\downarrow}$\\
Lifted prod. & \code{39}{3}{3} & $1.19\!\times\!10^{-1}$ & $9.02\!\times\!10^{-2}$ & $\boldsymbol{24.0\downarrow}$\\
Random Tanner & \code{54}{11}{4} & $1.24\!\times\!10^{0}$ & $4.67\!\times\!10^{-1}$ & $\boldsymbol{62.5\downarrow}$\\
Random Tanner & \code{60}{2}{6} & $2.88\!\times\!10^{-2}$ & $1.20\!\times\!10^{-2}$ & $\boldsymbol{58.4\downarrow}$\\
Random Tanner & \code{108}{18}{4} & $1.10\!\times\!10^{0}$ & $4.61\!\times\!10^{-1}$ & $\boldsymbol{58.2\downarrow}$\\
\bottomrule
\end{tabular}
\end{table}

\subsection{Experimental Setup}
\label{sec:setup}
\label{sec:hyperparameters}
\paragraph{Benchmarks and noise.}
\label{sec:noise-models}
We assess all eleven released \AS{} circuits~\cite{AlphaSyndrome2026}
and all eight \PH{} circuits~\cite{PropHunt2026}.
Every comparison uses the Brisbane noise model released by \AS{},
including comparisons with \PH{} circuits. We choose this model because
the active and idle error rates are derived from IBM Brisbane
calibration data~\cite{AlphaSyndrome2026}. The distinct rates represent
different fault costs for active and idle qubits.

Each comparison matches the code, noise, decoder, and memory duration.
The released model depolarizes active and idle ancillas after each CNOT
layer with probabilities $7.43267\times10^{-3}$ and
$5.24398\times10^{-3}$, respectively. Data errors arise by propagation.
Preparation, Hadamard gates, and readout are ideal, meaning that the
model assigns no faults to these operations.
The \AS{} suite uses one noisy extraction round with ideal boundary
measurements. The \PH{} suite retains its $d+1$ rounds and temporal
detectors. We report these two memory protocols separately.

We use Stim~\cite{Gidney2021stim}, MWPM for surface codes,
BP-OSD for colour codes (including hyperbolic colour), and BP-LSD
for \PH{} quantum LDPC codes~\cite{Higgott2022pymatching,Higgott2024stimbposd,Hillmann2025lsd}.
Decoder probabilities and connectivity are derived from each circuit.
Experiments run on a six-core Intel Core i5-12400.

\paragraph{Independent assessment.}
\label{sec:sig}
Final assessment uses fresh target-noise samples after search fixes
the output. We report whole-window $\LER=p_Z+p_X$ and its sampling
uncertainty. Fixed-shot assessments use Clopper--Pearson
intervals~\cite{ClopperPearson1934} with Bonferroni correction for
simultaneous comparisons~\cite{NISTBonferroni}. Ten \AS{} comparisons
use nominal two-standard-error intervals under adaptive stopping, with
approximate coverage. When intervals overlap, the LER ordering is
uncertain. Appendix~\ref{app:assessment} gives the interval construction
and shot budgets.

\subsection{RQ1: Schedule Effectiveness}
\label{sec:results}
\textbf{LER reductions average \meanRed\% versus \AS{} and 71.7\%
versus \PH{}.} For each suite we average
$100(1-L_{\mathrm{FS},i}/L_{\mathrm{base},i})$ equally over its circuits,
using unrounded LERs and including cases with overlapping intervals.
Table~\ref{tab:rq1-compact} reports all 19 comparisons.

Figure~\ref{fig:full-pipeline} shows how gate ordering changes fault
propagation in two Steane schedules. The same fault on ancilla $X_C$
produces a logical failure in the first policy candidate and a harmless
stabilizer after decoding in the delivered schedule.

All eleven \AS{} point estimates improve, by up to \peakRed\%.
Confidence bounds establish lower LER in seven \PH{} comparisons
(58.2\%--98.9\% estimated reductions). The lifted-product
\code{39}{3}{3} code, abbreviated LP39, has a 24.0\% point reduction
whose uncertainty interval includes no improvement.

RQ1 compares delivered schedules with unequal prior search effort.
The \AS{} cases use four 80-batch PPO runs each, with eight 160-batch
runs for the 30-qubit case. For the \PH{} suite, we select among fresh
and previously generated candidates using separate validation before
final assessment.
RQ2 next examines fresh searches with matched time caps.

\subsection{RQ2: Scalability Study}
\label{sec:scalability}
\label{sec:speed}
\paragraph{Learning progress and feedback cost.}
\label{sec:learning-curves}
\begin{figure}[tbp]
  \centering
  \includegraphics[width=\columnwidth]{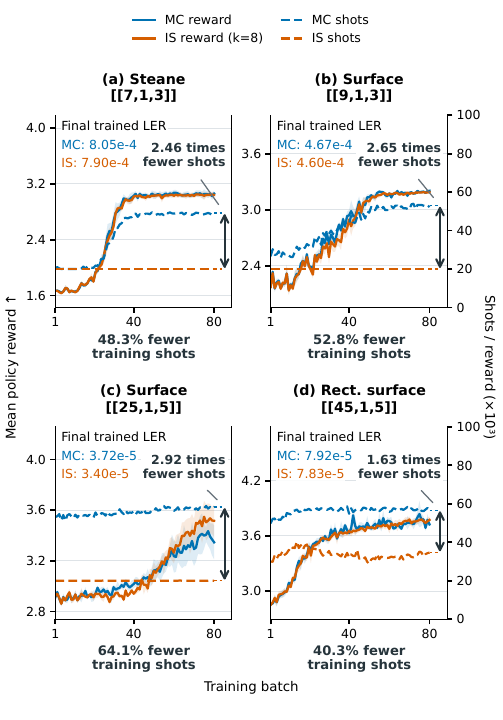}
  \caption{Learning and feedback cost under Brisbane noise.
    Solid curves show mean LER reward, with $\pm1$ standard error across
    four paired runs. Dashed curves show shots per schedule across both
    memories on a common scale. Arrows compare final-batch shots.
    Bold percentages give total training-shot savings.
    Appendix~\ref{app:feedback-cost} gives training and assessment budgets.}
  \Description{Four panels compare Monte Carlo and importance-sampled
    learning on Steane [[7,1,3]], surface [[9,1,3]], surface [[25,1,5]],
    and rectangular surface [[45,1,5]]. Each panel includes learning
    curves, dashed shot-cost curves, final trained LER annotations,
    and a double-headed arrow. The final-batch shot ratios are 2.46,
    2.65, 2.92, and 1.63; total training-shot savings are 48.3, 52.8,
    64.1, and 40.3 percent, respectively.}
  \label{fig:learning-curves}
\end{figure}

We extend the Steane \code{7}{1}{3} comparison to surface
\code{9}{1}{3}, \code{25}{1}{5}, and \code{45}{1}{5} circuits from
the \AS{} baseline suite. Each circuit uses four paired runs,
80 PPO batches, and fixed $k=8$ for IS. Both arms use identical
direct-MC candidate selection and omit local improvement to isolate
training feedback.

In Figure~\ref{fig:learning-curves}, importance sampling uses
48.3\%, 52.8\%, 64.1\%, and 40.3\% fewer total training shots,
respectively. The final-LER intervals overlap on all four circuits.
The comparison measures sampling cost; it does not match uncertainty
at every training evaluation or isolate wall-clock speedups.
These four cases span 7--45 data qubits. A 30-qubit hyperbolic case
in the artifact has no training-shot reduction under the same setting.

\paragraph{Increasing code distance.}
\label{sec:surface-distance}
\sysname{} has confirmed lower LER than \AS{} at every tested distance
from $d=3$ through $d=15$. The point reduction at $d=15$ is \textbf{97.8\%}.
We test rotated surface codes at $d=3,5,\ldots,15$ using one noisy
Brisbane round, ideal boundaries, and MWPM. Each search has one
evaluator worker and an elapsed-time cap of 10 minutes through $d=9$
and 30 minutes thereafter, including calibration and selection.
Each method has one search run per distance. The amplification factor was retuned
after the initial study.

Figure~\ref{fig:surface-distance} shows independent target-noise
assessments with 95\% simultaneous bounds covering the registered
circuits and all permitted stopping points. The dashed \PH{} curve
shows the starting circuits, as in Figure~\ref{fig:intro-contours};
these searches returned no optimized schedules within the time caps.
At $d=9,11$, the bounds
support at least 98.3\% and 98.9\% lower LER than \AS{}, respectively.
At $d=15$, \sysname{} has three failures and $\LER=3.623\times10^{-9}$.
The \AS{} circuit has 126 failures and $\LER=1.678\times10^{-7}$.
The LER ratio's upper bound, 0.999881, confirms the advantage's direction with
limited precision on its magnitude.

These bounds describe sampling uncertainty for the assessed circuits;
they do not measure variation across repeated searches.
Appendix~\ref{app:distance-assessment} gives all sample counts, stopping
rules, and assessment procedures.

\begin{figure}[tbp]
  \centering
  \includegraphics[width=\columnwidth]{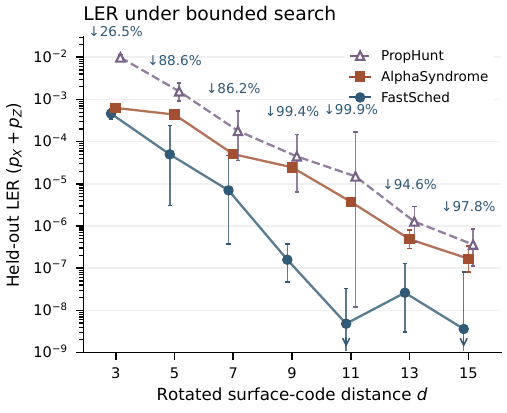}
  \caption{Independently assessed surface-code LERs: Brisbane
    ancilla noise, one noisy round, ideal boundaries, PyMatching.
    Search caps are 10 minutes through $d=9$ and 30 minutes thereafter.
    The dashed \PH{} curve shows starting circuits; its searches returned
    no optimized schedules within these caps.
    Whiskers show simultaneous 95\% bounds across all registered circuits
    and permitted assessment stopping points. Arrows at $10^{-9}$ mark lower bounds
    below the plot range. Down-arrow labels give percentage reductions,
    $100(1-\LER_{\mathrm{FastSched}}/\LER_{\mathrm{AlphaSyndrome}})$,
    from the point estimates. At $d=15$, the estimated reduction is
    97.8\%; the bounds establish the LER ordering.}
  \Description{FastSched, AlphaSyndrome, and PropHunt LERs at surface-code distances
    three through fifteen. FastSched is confirmed better through
    fifteen versus AlphaSyndrome. PropHunt is shown with a purple dashed line
    and hollow triangles; these are starting circuits from searches
    that returned no optimized schedules. Sampling gives positive estimates at distances nine
    and eleven, with 101 and 15 observed errors, respectively.
    Down-arrow labels are centered at distances three, five, seven,
    nine, eleven, thirteen, and fifteen and show 26.5, 88.6, 86.2, 99.4,
    99.9, 94.6, and 97.8 percent reductions, respectively.
    These reductions are computed from point estimates.
    At distance fifteen, three versus 126 errors give a 97.8
    percent point reduction and a narrowly resolved confidence comparison.
    FastSched at distance fifteen has 828.6 and 828.0 million
    shots in the two memories; AlphaSyndrome has 751 million per memory.
    Search caps are ten minutes through distance nine and
    thirty minutes thereafter.}
  \label{fig:surface-distance}
\end{figure}

\paragraph{Decreasing physical noise.}
Lower physical error rates make failures harder to observe during
search. We uniformly reduce both Brisbane probabilities and run fresh
searches at each noise level, since the best schedule can change with
noise strength. Figure~\ref{fig:scalability-wall} shows surface $d=5$
and LP39 alongside the random quasi-twisted code \code{54}{11}{4}
(RQT54), all at strengths $1/2$, $1/4$, $1/8$, $1/16$, and $1/32$.
Each circuit retains $d+1$ extraction rounds. \sysname{} and \AS{} receive two
search runs per setting with a 450-second ceiling, including setup
and selection, using one numerical thread on the same host.
Final evaluation uses independent samples and has no time cap.
For each method and setting, we display the lowest LER among assessed
outputs. Pale bars show one assessed run; other bars show the better
of two. Daggers mark ratios involving a pale bar. All search outputs
and assessment results are retained in the artifact.

\begin{figure}[t]
  \centering
  \includegraphics[width=\columnwidth]{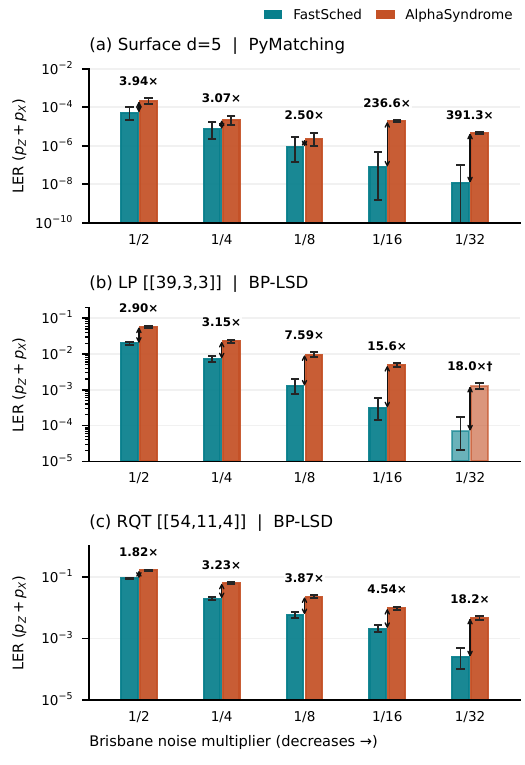}
  \caption{Low-noise scalability after 450-second searches.
    Bars show the lower assessed LER from two runs, or one run when pale.
    Daggers flag ratios involving pale bars. Arrow labels give
    \AS{}/\sysname{} LER ratios. Whiskers show simultaneous 95\% bounds
    from direct Monte Carlo.}
  \Description{Three stacked log-LER bar panels compare FastSched
    and AlphaSyndrome under decreasing Brisbane
    noise at one half, one quarter, one eighth, one sixteenth and one
    thirty-second strength. Double-headed arrows and factor annotations
    show the measured LER ratios. Pale bars denote one assessed run;
    other bars show the lower LER from two assessed runs. Daggers mark
    ratios involving a pale bar. All comparisons use direct sampling.}
  \label{fig:scalability-wall}
\end{figure}

\textbf{The largest measured LER advantages occur at lower physical noise.}
For surface $d=5$, the ratio of \AS{} LER to \sysname{} LER reaches
$236.6\times$ at $1/16$ strength and $391.3\times$ at $1/32$.
The corresponding \sysname{} LERs are $7.81\times10^{-8}$ and
$1.17\times10^{-8}$. Their wide bounds reflect five and three observed
failures, respectively.

Across the displayed noise range, LP39's ratio increases from
$2.90\times$ to $18.0\times$, and RQT54's ratio increases from
$1.82\times$ to $18.2\times$. These comparisons describe the best
assessed schedules from one or two runs per setting. Appendix~\ref{app:low-noise}
gives all measurement counts, sampling budgets, and assessment details.

\subsection{RQ3: Ablation Study}
\label{sec:ablation}
\label{sec:selection-eval}
Across eight circuits and four paired runs, candidate re-evaluation
gives the largest measured effect: \textbf{66.0\% lower mean normalized
LER} than omitting this stage. Larger shortlists give a 64.4\% reduction,
and importance-sampled feedback gives 8.1\%. Measurement and bootstrap
intervals support all three effects. With this short budget, the full
method averages 13.0\% above the released baseline LERs.
A separate 19-circuit study attributes 21.0\% of the final advantage to
local improvement. Appendix~\ref{app:ablation} gives the normalization,
all ten comparisons, and their uncertainty analyses.

\FloatBarrier
\section{Related Work}
\label{sec:related}
\paragraph{Fault-tolerant circuit design.}
Gate ordering controls how ancilla faults propagate to data qubits.
Beverland et al. formalize fault distance and hook faults for stabilizer
channels~\cite{Beverland2024}. Surface-code schedules use this connection
to preserve distance~\cite{Tomita_2014}. Kishony and Fowler orient hooks
diagonally~\cite{Kishony2025}, while Hirai et al. interleave $X$ and $Z$
checks for regular surface-code layouts~\cite{Hirai2025}.
Pacenti et al. address hook errors through decoder design~\cite{Pacenti2025}.

\paragraph{Automatic circuit synthesis.}
\AS{} and \PH{} demonstrate the value of searching for better extraction
circuits~\cite{AlphaSyndrome2026,PropHunt2026}.
\AS{} explores ordering and parallelism through MCTS guided by decoded
LER. \PH{} uses MaxSAT to find minimum-weight fault patterns that create
decoding ambiguities, then reorders or reschedules gates to resolve them.
\AS{} supports commuting stabilizers, while \PH{} targets CSS codes.
\sysname{} learns a policy that constructs schedules one check at a time.
Our importance-sampled evaluator supplies LER rewards for each complete
circuit.

\paragraph{Rare-event evaluation.}
Evaluating reliable circuits becomes expensive as logical failures grow
rare. Bravyi and Vargo use splitting to estimate rare logical-error
probabilities~\cite{BravyiVargo2013rare}. Mayer et al. extend splitting
to circuit-noise simulations~\cite{Mayer2025rare}.
ScaLER combines stratified fault injection with
extrapolation~\cite{YePalsberg2026}. Fail fast models the failure spectrum,
counts minimum-weight failing configurations, and uses splitting for
QLDPC codes~\cite{Beverland2025failfast}.
Our evaluator uses importance weights for each candidate circuit, so
rewards reflect the target noise even when failures are sampled under
amplified noise.

\paragraph{Learning with rare events.}
Rare-event sampling can also support learning.
Frank et al. oversample rare transitions and correct policy-value updates
with importance weights~\cite{Frank2008rare}. Corso et al. use RL to adapt
sampling distributions for rare-event estimation~\cite{Corso2022rare}.
In \sysname{}, the learned policy chooses gate orders. A separate
calibration chooses the amplification factor
(Section~\ref{sec:calibration}). This separation lets the agent learn
schedules from LER feedback at the target noise level.

\section{Conclusion}
\label{sec:conclusion}
RL with amplified-noise feedback improves
the quality and scalability of syndrome-extraction scheduling.
Across benchmark circuits, \sysname{} achieves average LER reductions
of \meanRed\% versus \AS{} and 71.7\% versus \PH{}
(Table~\ref{tab:rq1-compact}). With matched 30-minute
search limits, \sysname{} reaches a measured LER of $3.62\times10^{-9}$
at $d=15$, an estimated 97.8\% reduction versus \AS{}
(Section~\ref{sec:surface-distance}).

The key is to make rare logical failures useful for learning.
Our RL agent explores combinations of gate orders and learns which
choices yield lower LER. Noise amplification produces failures more
often, giving the agent useful feedback during search. Importance
weights convert these samples into LER estimates at the target noise
level.

These results support combining RL and importance sampling for
syndrome-extraction scheduling when direct failure samples are scarce.

\section{Artifact}
\label{sec:artifact}
The code, experimental results, and instructions for reproducing all
results reported in this paper are available in our anonymous artifact:
\url{https://anonymous.4open.science/r/FastSched-B5FE/}.

\paragraph{Generative AI disclosure.}
A large language model assisted with manuscript text and experiment
scripts. The authors produced and verified all results.

\clearpage

\phantomsection
\label{page:references-start}
\bibliographystyle{ACM-Reference-Format}
\bibliography{refs}

\clearpage
\appendix
\numberwithin{equation}{section}
\numberwithin{figure}{section}
\section{Supplementary Scheduling Details}
\label{app:scheduling}

The main paper explains the learning loop and the three-stage workflow.
This appendix records the action tables, PPO details, exploration
settings, and post-training selection procedure needed to reproduce
the search.

\subsection{Training Implementation Details}
\label{app:training-details}

\paragraph{Training limits.}
The training decision in Figure~\ref{fig:scheduling-workflow} checks
preset search budgets. Before each batch, the trainer checks the batch
limit, elapsed-time limit, and any configured total shot budget.
Another batch begins when these limits allow it. Otherwise, that run
ends. After all training runs finish, their saved pools pass to re-evaluation.
The standard \AS{} comparison allows 80 batches per run, with 30
candidate evaluations per batch. The extended hyperbolic search allows
160 batches per run. Appendix~\ref{app:selection} gives the run counts.
The scalability study uses the elapsed-time caps in
Section~\ref{sec:scalability}.

\paragraph{Action tables and network.}
Section~\ref{sec:mdp} defines the state and gate-order decisions.
The following construction fixes the available actions before training.

A table $\mathcal A_i$ lists the candidate CNOT orders for check $i$.
The actor selects an entry in this table.
For $w_i!\leq32$, the table contains every permutation in lexicographic
order. Larger checks use a table of 32 entries.
We sort the support, append each cyclic rotation followed by its
reversal, remove duplicates, keep at most 32 entries, and fill any remaining slots with
distinct random permutations. The table-generation seed is fixed at
zero, so every PPO run uses the same tables.

The cap concentrates training on a small set of actions.
A weight-six check has 720 possible orders, and a weight-eight check has
40,320. A 32-entry table lets a fixed training budget revisit each action
more often. We construct the table before any circuit evaluation.
Training evaluates the selected orders, and local improvement extends
the table with further permutations (Section~\ref{sec:selection}).

The neural network converts the state into action probabilities and a
value prediction. Its two shared, fully connected hidden layers each
contain 128 neurons. Each neuron applies the hyperbolic tangent
activation, $\tanh$, to its weighted input.

The actor head scores $\max_j|\mathcal A_j|$ entries and masks invalid
entries before sampling. The critic head predicts $V_\theta(s_i)$.
The parameter vector $\theta$ contains the weights and biases of the
shared layers and both heads.
In the Steane example, $X_B$ has 24 valid orders, and $a_2=11$ selects
$\mathcal A_2[11]=(3,7,1,5)$.
Recording this order in $H_B$ and advancing $c_i$ creates the state for $X_C$.

\paragraph{Credit assignment and policy updates.}
Equation~\eqref{eq:gae} propagates the final circuit's reward through
the recorded gate-order decisions.
The prediction error $\delta_i$ compares the observed reward and next
state's predicted value with the current prediction.
The backward recursion passes later prediction errors to earlier choices,
so PPO can favor actions associated with better outcomes.
For example, increasing only the terminal reward by $\Delta$ adds
$0.95^{m-i}\Delta$ to the advantage of decision $i$.

PPO controls the size of a policy update through the clipped objective
in Eq.~\eqref{eq:ppo-clip}.
Clipping replaces ratios below $0.8$ by $0.8$ and ratios above $1.2$
by $1.2$. Ratios inside the interval stay unchanged.
Angle brackets denote the mean over policy decisions in the batch.
A positive advantage favors increasing the chosen action's probability.
A negative advantage favors decreasing it.
Clipping caps the improvement credited by the objective: at ratio $1.2$
for a positive advantage and $0.8$ for a negative advantage.
Taking the smaller term also retains the penalty for changes that oppose
the advantage. This controls the incentive for large updates while
allowing several updates from one collected batch.

\paragraph{Evaluation overhead.}
\label{sec:cost}
We reuse circuits and decoder workers and deduplicate BP-OSD syndromes
to reduce evaluation overhead. Section~\ref{sec:speed} measures
reductions in sampled shots.

\subsection{Coordinated Exploration}
\label{app:exploration}
Gate orders can work well as a pattern shared across several checks.
The policy can learn such dependencies from its history input.
To expose useful patterns early, we also construct candidates with
\emph{shared-index exploration}. We draw one table index for each group
of checks with the same Pauli type and weight and apply that index to
every check in the group.

For example, the relative order $(3,4,1,2)$ gives Steane orders
$(3,4,1,2)$ for $X_A$, $(5,7,1,3)$ for $X_B$, and $(5,6,3,4)$ for $X_C$.
All three checks visit the third, fourth, first, and second qubits of
their sorted supports. Independent uniform draws choose a shared relative
order for these three checks with probability $1/24^2$.
Shared-index exploration tests such patterns directly.
Complete permutation tables and the rotation/reversal entries use
consistent relative orders. Random entries in capped tables are
specific to each check.

In the standard 30-candidate batch, 22 candidates come from the actor
and eight come from auxiliary exploration.
For each auxiliary candidate, we choose shared-index exploration or
independent uniform actions with equal probability.
Both kinds of auxiliary candidates enter the pool and update visit counts.
The GAE and PPO batch contains the 22 actor-generated trajectories.
Their recorded action probabilities supply the denominator of $\rho_i$.
Auxiliary candidates can become the final output through selection and
can influence later exploration bonuses through visit counts. Section~\ref{sec:ablation} measures the effects of these strategies.

\subsection{Post-Training Selection}
\label{app:selection}

Training produces a distribution over schedules and a pool of promising
circuits. Deployment requires one fixed circuit.
We use fresh samples to compare the saved candidates more precisely, then
fine-tune selected candidates through local changes.
Training spreads its budget across thousands of candidates.
Selection concentrates larger sample budgets on a small shortlist.
The shortlist size, local-search limits, and samples per evaluation
together control the additional compilation cost.
With local improvement enabled, the compiler compares the final schedule
from each local search and returns the one with the lowest LER estimate.
Disabling local improvement returns the best re-evaluated candidate.
Syndrome extraction repeatedly executes that fixed schedule.
The actor and critic remain fixed during selection.

\paragraph{Candidate re-evaluation.}
We form the shortlist from the saved training results.
Each run records a distinct schedule's first LER estimate and keeps at
most 300 candidates, ranked by that estimate. We merge the run-level
pools and deduplicate schedules. If several runs retained the same
schedule, its lowest stored LER determines its initial rank.
We then retain the 30 best-ranked candidates for fresh evaluation.

For the \AS{} suite, our configuration uses four independent training
runs, 80 batches per run, and 30 candidate evaluations per batch.
The product $4\times80\times30=9600$ counts evaluations from both policy
and auxiliary exploration, including repeated schedules.
The four run-level pools contain at most 1200 entries before
merging. We re-evaluate the top 30, retain the five with lowest fresh LER,
and evaluate those five again with larger budgets.
Thus, two re-evaluation passes use 35 candidate evaluations after training.
Fresh samples reduce the influence of favorable training estimates.
The smaller-budget component study in Section~\ref{sec:ablation} uses
24 candidates in the first pass and six in the second.
Its reduced-pool ablation retains two at each pass under the same
shot-budget ceilings. The extended hyperbolic RQ1 search uses eight runs
and 160 batches per run, with the same 30-to-five selection rule.

\paragraph{Local improvement.}
A \emph{starting schedule} is one of the highest-ranked circuits after
the second re-evaluation pass. Each selected schedule initializes a
separate local search. Multiple starting schedules let local improvement
explore different gate-order combinations.
A local trial changes one check's CNOT order, holds
the other orders fixed, reschedules the gates, and evaluates the resulting
circuit. At the start of each round, we add up to 24 fresh permutations
to each incomplete action table. We screen unvisited one-check changes
and re-evaluate the three best with larger, independent samples.

If check $i$ has $K_i$ available orders, a round screens at most
$\sum_i(K_i-1)$ neighbors. For Steane, this is $6(24-1)=138$ neighbors,
compared with $24^6$ combinations in the full order space.
For the \AS{} suite, local improvement refines two starting schedules
for up to five rounds each. At most 1380 neighbor screens are therefore needed
for Steane, plus confirmation and any repeated close comparisons.
Separate sample limits bound the cost of each evaluation.
We enable this stage in the main pipeline and make it configurable so
ablations can measure its contribution. Each starting schedule has its
own current schedule and set of visited changes.

Section~\ref{sec:selection-eval} reports paired local improvement on all
19 baseline circuits and component ablations on eight circuits.

Figure~\ref{fig:smc}(a,b) illustrates the gate-order change considered
by a local trial. For the Steane check $X_A=X_1X_2X_3X_4$, replacing
$(1,2,3,4)$ with $(1,3,2,4)$ changes the displayed hook from $X_3X_4$
to $X_2X_4$. During local improvement, we hold the other five check
orders fixed, reschedule the complete circuit, and compare the resulting
LER estimates.

\paragraph{Accepting a local change.}
Local improvement should favor changes whose estimated benefit exceeds
the sampling fluctuations in the comparison. Our acceptance rule uses
one estimated standard error of the LER difference as a margin.
For $b\in\{Z,X\}$, let $\widehat p_b(\sigma)=f_b/N_b$.
For $N_Z,N_X>1$, we use the sample-variance estimate for the summed LER:
\begin{equation}
u(\sigma)^2=
\frac{\widehat p_Z(\sigma)(1-\widehat p_Z(\sigma))}{N_Z-1}
+\frac{\widehat p_X(\sigma)(1-\widehat p_X(\sigma))}{N_X-1}.
\label{eq:ler-uncertainty}
\end{equation}
For the current schedule $\sigma_0$ and best re-evaluated neighbor
$\sigma'$, we accept a change when
\begin{equation}
\widehat{\LER}(\sigma_0)-\widehat{\LER}(\sigma')
>\sqrt{u(\sigma_0)^2+u(\sigma')^2}.
\label{eq:accept}
\end{equation}

For independent evaluations, the variance of a difference is the sum of
the two variances. Taking the square root gives the margin on the
right-hand side. A larger margin asks for a larger measured improvement
when either estimate is noisy.
For a positive improvement below the margin, we increase the budget and
re-evaluate both schedules up to preset limits.
After accepting a change, we evaluate the updated schedule again with
fresh samples. This new estimate becomes the current LER estimate for
the next round in Figure~\ref{fig:scheduling-workflow}.
The one-standard-error threshold is a search heuristic. Independent
assessment provides the evidence for the reported LER comparisons.

\paragraph{Stopping local improvement.}
The local-search decision in Figure~\ref{fig:scheduling-workflow} checks
that another round is allowed and untested one-check changes remain.
The \AS{} comparison permits five rounds per starting schedule.
Each round screens and confirms changes and accepts at most one.
A rejected change, the round limit, or exhausted untested changes ends
that local search. After every selected starting schedule has completed
its local search, the compiler ranks the refined schedules by their
latest LER estimates and fixes one output.
Selection and local improvement count toward compilation cost.
Independent assessment begins after the compiler has fixed its output
(Section~\ref{sec:sig}).

\subsection{Amplification Calibration}
\label{app:calibration}

Calibration chooses how strongly to amplify noise before training.
During search, the evaluator chooses the shots per candidate.
Figure~\ref{fig:evaluation-control} separates these decisions.

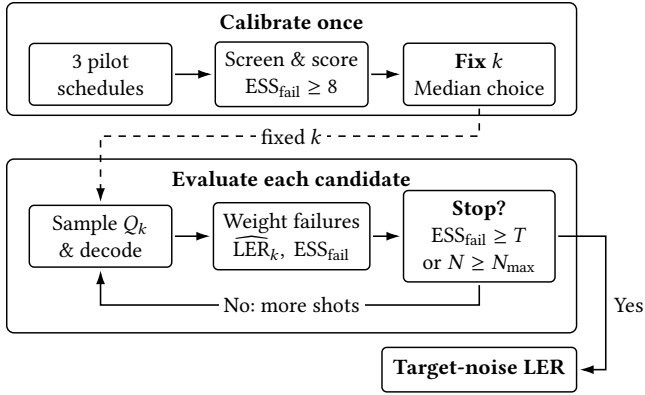
\begin{figure}[tbp]
\centering
\resizebox{\columnwidth}{!}{%
\begin{tikzpicture}[x=1cm,y=1cm,font=\small,
  block/.style={draw=black,rounded corners=2pt,fill=white,
    line width=.6pt,align=center,inner sep=4pt},
  flow/.style={-{Latex[length=2mm,width=1.3mm]},draw=black,
    line width=.65pt,shorten <=1pt,shorten >=1pt}]
\draw[black,rounded corners=3pt,line width=.6pt]
  (0,0) rectangle (8.2,-1.62);
\node[font=\small\bfseries] at (4.1,-.27) {Calibrate once};
\node[block,minimum width=2.05cm,minimum height=.82cm]
  (pilot) at (1.35,-1.03) {3 pilot\\schedules};
\node[block,minimum width=2.10cm,minimum height=.82cm]
  (score) at (4.1,-1.03) {Screen \& score\\$\ESS_{\mathrm{fail}}\ge8$};
\node[block,minimum width=2.18cm,minimum height=.82cm]
  (factor) at (6.8,-1.03) {\textbf{Fix $k$}\\Median choice};
\draw[flow] (pilot) -- (score);
\draw[flow] (score) -- (factor);

\draw[black,rounded corners=3pt,line width=.6pt]
  (0,-2.25) rectangle (8.2,-4.75);
\node[font=\small\bfseries] at (4.1,-2.53) {Evaluate each candidate};
\node[block,minimum width=2.05cm,minimum height=.85cm]
  (sample) at (1.35,-3.36) {Sample $Q_k$\\\& decode};
\node[block,minimum width=2.10cm,minimum height=.85cm]
  (weight) at (4.1,-3.36)
  {Weight failures\\$\widehat{\LER}_k,\ \ESS_{\mathrm{fail}}$};
\node[block,minimum width=2.18cm,minimum height=1.15cm]
  (stop) at (6.8,-3.36)
  {\textbf{Stop?}\\$\ESS_{\mathrm{fail}}\ge T$\\or $N\ge\Nmax$};
\draw[flow,dashed] (factor.south) -- (6.8,-1.92) -- (1.35,-1.92)
  -- (sample.north);
\node[fill=white,inner sep=2pt] at (4.1,-1.92) {fixed $k$};
\draw[flow] (sample) -- (weight);
\draw[flow] (weight) -- (stop);
\draw[flow] (stop.south) -- (6.8,-4.35) -- (1.35,-4.35)
  -- (sample.south);
\node[fill=white,inner sep=2pt] at (4.1,-4.35) {No: more shots};
\node[block,minimum width=2.8cm,minimum height=.64cm,
  font=\small\bfseries] (result) at (6.8,-5.28) {Target-noise LER};
\draw[flow] (stop.east) -- (8.62,-3.36) -- (8.62,-5.28)
  -- (result.east);
\node[anchor=west,inner sep=1pt] at (8.70,-4.35) {Yes};
\end{tikzpicture}}
\caption{Amplification calibration and adaptive shot allocation.}
\Description{A calibration group screens and scores three pilot schedules
using effective failures, then fixes k from the median choice. A dashed
arrow carries k to the candidate-evaluation group. Sampling, decoding,
and weighted failures produce a target-noise LER estimate and an effective
failure count. Evaluation repeats until the effective failure target or
shot cap is reached, then returns the LER. The target-noise decoder stays
fixed throughout.}
\label{fig:evaluation-control}
\end{figure}

Amplification trades more frequent failures against greater variation in
importance weights and increased decoding cost.
Our default calibration tests $k\in\{1,2,3,5,8,12\}$ on three random
schedules before PPO training. Each test uses 2,000 shots per memory
circuit. The pilot estimates LER, its standard error $\widehat s_k$,
and the decoding work $D_k$.

A useful pilot estimate draws information from several failed samples.
The \emph{effective sample size} (ESS) of the failure contributions,
$\ESS_{\mathrm{fail}}$, expresses this information as a count of equally
weighted failures. Dominant weights reduce the count
(Section~\ref{sec:comparison}). For each pilot schedule, we retain
settings with $\ESS_{\mathrm{fail}}\ge8$. This screening heuristic
requires several effective failures before comparing estimates.

For $m$ equally weighted failures, every share is $1/m$, so Eq.~\eqref{eq:ess}
returns $m$. For example, four equal contributions give 4.
Weights $(7,1,1,1)$ instead give $100/52\approx1.92$.
Although both examples contain four failures, the second estimate
depends heavily on a single shot and calls for more sampling.
At $k=1$, $\ESS_{\mathrm{fail}}$ equals the number of observed failures.

Among the retained settings, we choose the smallest score
\begin{equation}
S(k)
=D_k\left(\frac{\widehat s_k}{\widehat{\LER}_k}\right)^2,
\label{eq:k-score}
\end{equation}
where $D_k$ counts syndromes actually decoded, with a minimum of one.
BP-OSD reuses results for repeated syndromes, so only distinct rows count.
The squared ratio estimates relative variance. At a fixed cost per shot,
doubling the samples doubles work and approximately halves variance.
Their product estimates work at a common relative precision, up to a
shared factor. The decoded-row count makes calibration reproducible
across machines.

If every setting fails the screen, the pilot selects $k=1$.
We take the median of the three pilot selections and hold this value
fixed during PPO, candidate re-evaluation, and local improvement in the
default pipeline. A supplied $k$ bypasses calibration.
Direct-MC ablations explicitly set $k=1$.
The component study in Section~\ref{sec:ablation} uses direct MC for all
candidate re-evaluation and local improvement, which isolates the effect
of amplification during training.

\subsection{Weighted Uncertainty for Local Selection}
\label{app:weighted-uncertainty}

Section~\ref{sec:amplification} supplies weighted LER estimates to the
same scheduling workflow. The local acceptance margin uses the variance
of these weighted observations.

Local acceptance uses the uncertainty of the weighted LER.
Let $Y_{b,h}$ denote shot $h$'s weighted failure contribution for memory
$b$, as defined in Eq.~\eqref{eq:is}.
With $\bar Y_b=N_b^{-1}\sum_hY_{b,h}$ and $N_Z,N_X>1$, the variance estimate is
\begin{equation}
u_k(\sigma)^2=
\sum_{b\in\{Z,X\}}
\frac{\sum_{h=1}^{N_b}(Y_{b,h}-\bar Y_b)^2}{N_b(N_b-1)}.
\label{eq:is-uncertainty}
\end{equation}
At $k=1$, this expression reduces to Eq.~\eqref{eq:ler-uncertainty}.
Local improvement substitutes $\widehat{\LER}_k$ and $u_k$ into
Eq.~\eqref{eq:accept}. The comparison and re-evaluation steps follow
Section~\ref{sec:selection}.

For a fixed number $N$ of independent shots in one memory, the
estimator's variance is
\begin{equation}
\operatorname{Var}_{Q_k}(\widehat p_{L,k})
=\frac{\operatorname{Var}_{Q_k}(W_kF_\sigma)}{N}.
\end{equation}
Theorem~\ref{thm:unbiased} applies to this fixed-budget estimate.
Adaptive stopping uses the observed failure contributions and can
introduce bias, so final assessment uses independent target-noise MC.

\section{Supplementary Evaluation Details}
\label{app:evaluation}

Section~\ref{sec:eval} presents the main findings. This appendix retains
the assessment procedures, sample counts, and component configurations.
The artifact records the schedules, seeds, source hashes, and full logs.

\subsection{Independent Assessment}
\label{app:assessment}

Final assessment measures the whole-window $\LER=p_Z+p_X$ with fresh
target-noise samples after each search has fixed its output.
Confidence intervals quantify the remaining sampling uncertainty.
For fixed shot counts, we use Clopper--Pearson binomial
intervals~\cite{ClopperPearson1934}. Each interval bounds a memory's
failure probability from its observed failures and shots. Adding the
two lower endpoints and the two upper endpoints gives bounds on
$p_Z+p_X$. To bound a ratio, we divide a numerator's lower bound by
the denominator's upper bound, and conversely for the upper ratio bound.
Bonferroni correction divides the allowed error probability among all
intervals in a declared comparison set. The resulting bounds cover
the set simultaneously with at least the stated
confidence~\cite{NISTBonferroni}.

The extended 30-qubit \AS{} case uses five million shots per memory.
Two declared LER ratios each receive 97.5\% bounds, giving at least
95\% simultaneous coverage. The \PH{} holdouts use fixed budgets of
$10^4$--$5\times10^5$ shots per memory with simultaneous 95\% bounds.
Ten earlier \AS{} cases stop at 3000 failures or $6\times10^7$ shots
per memory. Those cases use nominal two-standard-error intervals,
whose coverage is approximate under adaptive stopping.
An unresolved comparison means the assessment leaves the LER ordering
uncertain. Configurations, counts, schedules, source hashes, and the
complete statistical protocols are in the artifact.

\paragraph{Learning and feedback cost.}
\label{app:feedback-cost}
Figure~\ref{fig:learning-curves} uses Brisbane noise, BP-OSD for Steane,
and PyMatching for surface codes. The plotted LER rewards exclude the
exploration bonus. Each training batch contains 30 episodes. Training
targets 30 observed failures for direct Monte Carlo and 30 effective
failures for importance sampling, with a cap of 30,000 shots per memory.
Final LERs use two million fresh Monte Carlo shots per memory per schedule.

\subsection{Surface-Code Distance Study}
\label{app:distance-assessment}

\paragraph{Distance overview in the Introduction.}
Figure~\ref{fig:intro-contours} shows point estimates from the same
Brisbane experiments as Figure~\ref{fig:surface-distance}, which also
shows confidence bounds. All three methods are assessed at
$d=3,5,\ldots,15$ with one noisy extraction round, ideal boundaries,
and PyMatching. The dashed \PH{} curve represents its starting circuits:
these searches returned no optimized schedules within the time caps.
Lines connect measurements; they do not interpolate an analytic error
model. LER is $p_X+p_Z$ for the two memory experiments.
At $d=15$, \sysname{} has measured LER $3.623\times10^{-9}$, an
estimated 97.8\% reduction versus our \AS{} run. Simultaneous 95\%
bounds establish the ordering, with limited precision on the reduction.

\paragraph{Measured schedules.}
\sysname{} has confirmed lower LER than \AS{} at all seven tested
distances through $d=15$, with a \textbf{97.8\% lower measured LER at $d=15$}.
We test rotated surface codes at
$d=3,5,\ldots,15$ (9--225 data qubits), using one noisy Brisbane round,
ideal boundaries, and MWPM. All methods run on the six-core Intel Core
i5-12400 with one evaluator worker per search and one-thread numerical
libraries. Each search has an elapsed-time cap of 10 minutes at
$d=3$--9 and 30 minutes at $d=11$--15. Calibration and candidate selection
count toward these caps.

Each distance has one search seed. The amplification factor for
\sysname{} was retuned after the initial study. The reported results extend
independent assessment of these saved circuits, without further search,
under a cumulative cap of $10^{10}$ shots per method and distance.
The figure combines the assessed circuits at all seven distances.
The confidence bounds retain
95\% simultaneous coverage over all 32 registered circuits, both memories,
and all 20,000-shot exposure-grid looks.
Figure~\ref{fig:surface-distance} reports the independently assessed LERs
and their uncertainty.

At $d=9$,
\sysname{} has 56 failures in 637.0 million $X$ shots and 45 failures
in 636.2 million $Z$ shots. Its $\LER=1.586\times10^{-7}$ is
99.4\% below the \AS{} point estimate. The simultaneous bounds support
at least \textbf{98.3\% lower LER}. At $d=11$, ten failures in
3.0978 billion $X$ shots and five in 3.1026 billion $Z$ shots give
$\LER=4.840\times10^{-9}$, a 99.9\% point reduction. The simultaneous
bounds support at least \textbf{98.9\% lower LER} than \AS{}.
The $d=9$ stopping target was 100 errors
and a resolved comparison. The $d=11$ target was amended to 15 errors
during assessment. In-flight samples are retained, and the same
simultaneous confidence procedure covers these stopping choices.

At $d=13$, \sysname{} has 21 failures with 800.2 million shots per
memory, giving $\LER=2.624\times10^{-8}$. \AS{} has 243 failures
with 500 million shots per memory, giving $4.86\times10^{-7}$.
This is a \textbf{94.6\% point reduction}, with at least
\textbf{56.3\% lower LER} under the simultaneous 95\% bounds.
At $d=15$, the stopping rule allowed assessment until the simultaneous
intervals separated in either direction or the cumulative shot caps
were reached. \sysname{} has zero failures in 828.6 million $X$ shots
and three in 828.0 million $Z$ shots. \AS{} has zero and 126 failures
with 751 million shots per memory. Their LERs are
$3.623\times10^{-9}$ and $1.678\times10^{-7}$, a
\textbf{97.8\% point reduction}. The simultaneous FastSched/\AS{}
ratio upper bound is 0.999881, confirming the direction of the advantage
with limited precision on its magnitude. All completed batches are included.

These results concern the selected circuits rather than
variation across independent search seeds.

\subsection{Decreasing Physical Noise}
\label{app:low-noise}

\paragraph{Measurements in Figure~\ref{fig:scalability-wall}.}
For surface $d=5$, the ratio of \AS{} LER to \sysname{} LER reaches
$236.6\times$ at $1/16$ and $391.3\times$ at $1/32$.
At these strengths, \sysname{} has direct estimates
$7.81\times10^{-8}$ (5 failures, 64 million samples per basis)
and $1.17\times10^{-8}$ (3 failures, 256 million samples per basis),
respectively, using PyMatching. The wide uncertainty bounds reflect
the small numbers of observed failures.

For LP39, the ratios at $1/2$, $1/4$, and $1/8$ are
$2.90\times$, $3.15\times$, and $7.59\times$.
At $1/16$, both runs are assessed for each method. The best observed
LERs are $0.00031$ and $0.004825$, a $15.6\times$ ratio
(62 versus 965 failures, 200,000 samples per basis, BP-LSD).
The single-run comparison gives $18.0\times$ at $1/32$. Here, direct LERs are
$7.00\times10^{-5}$ for \sysname{} and $1.26\times10^{-3}$ for
\AS{} (28 versus 503 failures, 400,000 samples per basis, BP-LSD).
For LP39, the measured LER advantage increases at each displayed
halving of the noise strength.

RQT54 provides a second QLDPC comparison across the same noise range.
At $1/2$, both paired runs favor \sysname{}. The best observed LERs are
$0.08899$ and $0.16152$, a $1.82\times$ ratio (8,899 versus 16,152
failures, 100,000 samples per basis, BP-LSD).
Two runs per method are assessed at all five strengths. The lowest-LER ratios
are $3.23\times$ at $1/4$, $3.87\times$ at $1/8$, and $4.54\times$
at $1/16$. At $1/8$, the selected LERs are $0.00584$ and $0.02261$
(584 versus 2,261 failures, 100,000 samples per basis, BP-LSD).
At $1/16$, the LERs are $0.002085$ and $0.00946$ (417 versus 1,892
failures, 200,000 samples per basis, BP-LSD). Both comparisons have
separated simultaneous bounds. The $1/32$ ratio is
$18.2\times$, with selected LERs $0.00025$ and $0.00454$
(50 versus 908 failures, 200,000 samples per basis, BP-LSD),
using the lower LER from two assessed runs for each method.
All five displayed strengths have at least one assessed output per
method. The artifact records every assessed run.

These comparisons describe the best assessed schedules from one or two
runs per setting.

\paragraph{Initial noise ladder.}
\begin{table}[tbp]
\centering\small
\setlength{\tabcolsep}{7pt}
\caption{Direct Monte Carlo comparisons at three Brisbane strengths.
Ratios are \AS{}/\sysname{} LER, using each method's lower LER from
two assessed runs. Larger ratios favor \sysname{}.
Bold ratios have separated simultaneous 95\% bounds under the
original 54-search correction (Appendix~\ref{app:low-noise}).}
\label{tab:low-noise-detail}
\begin{tabular}{@{}lrrr@{}}
\toprule
Brisbane multiplier & $1/2$ & $1/4$ & $1/8$\\
\midrule
\multicolumn{4}{l}{\textbf{\AS{}/\sysname{} LER ratio}}\\
Surface $d=3$ & 1.33 & \textbf{2.05} & 1.23\\
Surface $d=5$ & \textbf{3.94} & 3.07 & 2.50\\
LP39 & \textbf{2.90} & \textbf{3.15} & \textbf{7.59}\\
\bottomrule
\end{tabular}
\end{table}

Lower physical error rates make failures harder to observe during
search. To examine this setting, we run fresh searches on surface codes
at $d=3,5$ and on LP39. We multiply both Brisbane error probabilities
by $1/2$, $1/4$, or $1/8$, preserving their ratio and fault locations.
These multipliers define Brisbane strength. Every circuit retains
$d+1$ extraction rounds.

The \sysname{}--\AS{} comparison uses two attempts per method for each
code and noise level, giving 36 searches. The confidence correction
retains all 54 attempts in the original registration.
Each attempt has a 450-second elapsed-time ceiling on the six-core
Intel Core i5-12400, including setup and candidate selection. Numerical
libraries are limited to one thread. These budgets describe search
on this host.

We assess each eligible output with independent, fixed-shot direct MC.
Surface-code assessments use 1, 4, and 16 million shots per memory
at half, quarter, and eighth strength. LP39 uses 100,000 shots per
memory at each level. All 36 comparison searches have assessed outputs.
For each method, code, and noise level, we report the output with the
lowest measured LER across the two attempts.

Each interval bounds the whole-window LER of a fixed schedule.
The intervals have at least 95\% simultaneous coverage.
The correction reserves an error probability of
$0.05/54$ per registered attempt, including attempts with no output.
Keeping both attempts in this correction preserves simultaneous
coverage when we select the displayed result.

Table~\ref{tab:low-noise-detail} gives all nine \AS{}/\sysname{} LER
ratios. A ratio above one favors \sysname{}.
\sysname{} has lower point LER in all nine comparisons, with
\textbf{five resolved improvements}. LP39's ratio rises from
$2.90\times$ to $3.15\times$ to $7.59\times$ as noise falls.
The $d=3$ surface-code ratio first rises and then falls, while the $d=5$
ratio decreases across the three levels. The increasing advantage
therefore appears in LP39 within this experiment.

This study describes the best observed schedules from two attempts.
More independent search seeds would be needed to estimate typical
optimizer performance and establish how the advantage changes with noise.

\paragraph{Lower-noise extensions and independent validation.}
The extension adds 24 registered surface-$d=5$ attempts at strengths
$1/16$, $1/32$, $1/64$, and $1/128$, and 18 QLDPC attempts on LP39,
RQT60 \code{60}{2}{6}, and RQT54 at strengths $1/128$, $1/128$,
and $1/2048$, respectively. Each setting has two attempts per method.
Searches retain the 450-second ceiling; evaluation has no wall-clock cap.
Surface direct assessments use 64 million, 256 million, one billion,
and two billion samples per basis at the four strengths.
The remaining assessments outside the displayed noise range are archived.

The QLDPC study assesses each eligible saved schedule with 500,000
importance samples per basis at two fixed factors: $k=16,64$ for LP39
and RQT60, and $k=256,1024$ for RQT54. Both proposals must pass the
registered effective-failure and weight diagnostics, and their LERs
must agree within three combined standard errors. The primary factor
is reported, never the lower proposal estimate. All four LP39 schedules
pass. Direct controls use 100,000 samples per basis: \sysname{} versus
\AS{} has 0 versus 103 failures for one seed and 1 versus 65 for the
other. Both orderings have separated simultaneous bounds. Initial
RQT60 estimates fail precision checks, and RQT54's \AS{} cross-checks
also fail; these outcomes remain archived as unresolved.

A separate validation extension evaluates the same four saved RQT54
schedules with one million samples per basis at each of $k=256,512$,
plus one million direct samples per basis. All fixed samples are used.
All four schedules pass both importance checks. The primary LERs for
the two \sysname{} seeds are $1.24713\times10^{-7}$ and
$9.42343\times10^{-8}$, with 7,764 and 6,213 amplified failures;
the \AS{} estimates are $5.94538\times10^{-5}$ and
$4.81796\times10^{-5}$, with 24,159 and 25,238 amplified failures.
Primary relative standard errors are 2.2\% and 2.4\% for \sysname{}
and 1.7\% and 2.0\% for \AS{}. The second \sysname{} seed's
cross-factor discrepancy is 2.97 combined standard errors, within the
registered three-standard-error criterion. This marginal agreement
is retained explicitly. No samples from the earlier assessment are
pooled with the extension, and no selection is made between old and
new estimates. Four corresponding direct controls have 0 versus 62
and 0 versus 46 failures; they establish both LER orderings but cannot
certify the \sysname{} rare-event rate below $10^{-8}$.

The direct intervals in Figure~\ref{fig:scalability-wall} use
$\alpha=0.05/148$, covering the initial 54 attempts, the 24 surface
and 18 QLDPC attempts, four additional RQT54 direct controls, and
48 additional noise-ladder attempts.
This includes missing outputs and both search seeds. The separate
importance-sampling assessments above use approximate $\pm1.96$
standard-error intervals without a joint finite-sample coverage guarantee;
they are not plotted in Figure~\ref{fig:scalability-wall}. The plans, circuit
and candidate hashes, seeds, per-block sufficient statistics, and
both validation proposals are retained in the artifact. Independent
100,000-sample blocks use fixed seeds; all completed blocks are pooled
within the same circuit and factor. RQT60 validation was discontinued
before its planned sample count. Its samples are retained and do not
contribute a plotted estimate.

The additional ladder covers RQT54 at $1/2$, $1/4$, $1/8$, $1/16$, $1/32$, and
$1/64$, and LP39 at $1/16$ and $1/32$. Each setting registers two
fresh searches per method with the same 450-second ceiling and no
warm start across noise strengths. Direct evaluation uses 100,000
samples per basis for RQT54 at the first three levels, 200,000 for
RQT54 at $1/16$ and $1/32$ and LP39 at $1/16$, and 400,000 for RQT54 at $1/64$
and LP39 at $1/32$. These counts are fixed before evaluation and
shared by all methods at each setting. All completed samples are used.
The displayed figure is restricted to strengths $1/2$ through $1/32$
for all three codes; lower-strength measurements remain in the archive.
LP39 at $1/32$ has one assessed output per method. The other displayed settings
have two assessed outputs per method. Pale bars denote one assessed
output; other bars show the lower measured LER from two. Daggers flag
ratios involving a pale bar. The confidence family retains both
registered runs even when only one has an assessed output.

\subsection{Component and Local-Improvement Studies}
\label{app:ablation}

The ablation study measures how the stages in
Figure~\ref{fig:scheduling-workflow} and the evaluator in
Figure~\ref{fig:evaluation-control} affect the final schedule.
We first compare component changes on eight circuits with
matched sampling budgets. A separate experiment measures the gain from
local improvement across all 19 baseline circuits.

\paragraph{Component comparisons.}
The eight-circuit study includes six \AS{} circuits and two \PH{}
circuits. For each circuit, four random seeds pair the full method with
four training variants. The variants remove the exploration bonus,
remove coordinated exploration, use direct-MC feedback, or use a fixed
shot count per training evaluation. Five configurations across eight
circuits and four seeds give 160 searches. Each search has a total
training ceiling of 393,216 shots per memory.

Six further comparisons change the stages after training, reusing the
full method's saved candidates. These changes narrow the candidate
shortlist, omit candidate re-evaluation, omit local improvement, remove
the acceptance margin, reduce confirmation sampling together with the
margin, or disable action-table expansion. Together, the four training
changes and six later changes give ten component comparisons.

All final outputs receive independent fixed-shot assessment.

To combine codes with different error rates, we divide each output's
LER by its released baseline's LER. A normalized LER below one means
the output improves on that baseline. We average these ratios across
the four seeds within each circuit, then equally across the eight
circuits. Let $R_{\rm full}$ be this mean for the full method and
$R_{-j}$ the mean after removing component $j$. The relative LER
reduction from including that component is
$100(1-R_{\rm full}/R_{-j})\%$.

\textbf{Candidate re-evaluation has the largest measured effect,
reducing mean normalized LER by 66.0\%.}
Retaining more candidates gives a 64.4\% reduction: the full method
re-evaluates 24 candidates and then six, while the narrow variant
retains two at both stages. Both variants have the same total selection
shot ceilings. Importance-sampled training feedback gives an 8.1\%
reduction with direct-MC selection in both variants.

We examine uncertainty from two sources separately. Measurement
intervals account for the finite assessment shots. Seed-bootstrap
intervals resample the four paired search outcomes within each circuit
to estimate search variability. Each analysis corrects for the ten
component comparisons. Both analyses favor these three components.
The other seven effects remain unresolved across seeds.
The component effects overlap and should be interpreted separately.
At this short training budget, $R_{\rm full}=1.1303$: even the full
method averages 13.0\% above the released baseline LER. The percentages
above quantify improvements over the corresponding ablated variants.

\paragraph{Contribution of local improvement.}
The separate 19-circuit experiment measures how much of the final
baseline advantage comes from the last search stage. We assess the
schedule before and after local improvement with independent samples.
The equal-circuit mean reduction from the baseline increases from
33.7790\% before local improvement to 42.7320\% afterward.
The additional 8.9531 percentage points account for
\textbf{21.0\% of the final advantage}:
\begin{equation}
 C_{\rm local}
 =100\,\frac{\text{additional reduction}}{\text{final reduction}}\,\%
 =100\,\frac{8.9531}{42.7320}\,\%.
 \label{eq:ablation-partition}
\end{equation}
Figure~\ref{fig:rq3-pie} shows this division of the measured gain.
The local-improvement share has an approximate simultaneous 95\%
measurement interval of $[18.5,23.4]\%$. This interval accounts for
assessment noise with the searched schedules held fixed.

Local search budgets differ between suites. \AS{} cases allow
up to five rounds from two starts. \PH{} cases add one round with
at most 256 neighbors. Their separate contribution estimates are
70.9\% and $-0.4\%$, respectively. The extra pass on \PH{} cases has
an unresolved effect. The artifact retains all null and negative results.

\begin{figure}[tbp]
  \centering
  \includegraphics[width=\columnwidth]{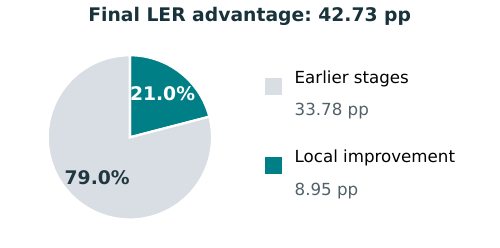}
  \caption{In a separate 19-circuit assessment, local improvement raises
    the average LER reduction from 33.78\% to 42.73\%, accounting for
    21.0\% of the final gain. The other 79.0\% precedes this stage.}
  \Description{A two-slice pie attributes 79 percent of the final
    baseline advantage to the endpoint before local improvement and 21 percent
    to the assessed local improvement pass.}
  \label{fig:rq3-pie}
\end{figure}

\clearpage
\end{document}